\documentclass[referee,sn-basic]{sn-jnl}

\usepackage{comment}
\usepackage{booktabs} 
\usepackage{graphicx} 
\usepackage{url}      
\usepackage{amsmath}  
\usepackage{amsthm}
\usepackage{multirow}
\usepackage{longtable}
\usepackage{float}
\theoremstyle{thmstyletwo}
\newtheorem{theorem}{Theorem}
\newtheorem{proposition}{Proposition}
\usepackage{breakurl}

\title{
Directional Evidence Guided Search-Space Reduction for Exact DAG Learning}

\author{\fnm{Upala Junaida} \sur{Islam}}
\email{uislam@asu.edu}
\author{\fnm{Abdelmonem} \sur{Elrefaey}}
\email{aelrefaey@asu.edu}
\author*{\fnm{Rong} \sur{Pan}}
\email{rong.pan@asu.edu}

\affil{
  \orgdiv{School of Computing and Augmented Intelligence},
  \orgname{Arizona State University},
  \city{Tempe},
  \state{AZ}
  \postcode{85281},
  \country{USA}}

\usepackage{setspace}
\usepackage{amsmath, amssymb}
\usepackage{xcolor}
\usepackage{soul}
\usepackage{natbib}
\usepackage[all]{nowidow}
\usepackage{algorithm2e}
\RestyleAlgo{ruled}
\allowdisplaybreaks

\begin{document}


\abstract{
Learning a directed acyclic graph (DAG) from observational data is a challenging combinatorial problem due to the exponential growth in the number of candidate parent-set configurations. Existing exact score-based methods often require computationally intensive combinatorial search, whereas constraint-based methods can become unreliable or computationally demanding as graph size and conditioning-set complexity increase. We develop a non-parametric hybrid framework, referred to as DECO (Directional Evidence-guided Configuration Optimization), that extracts dependency and directional evidence from observation data to construct admissible parent sets prior to exact optimization.
It reduces the optimization search space by eliminating empirically unsupported parent configurations while preserving flexibility for all plausible edge orientations. 
Theoretical analysis establishes an exponential reduction in the admissible parent-set configuration space and quantifies how bounded edge-level omission affects the probability of retaining the true parent structure.
Experiments on benchmark Bayesian networks and synthetic discrete and continuous DAGs demonstrate substantial search-space reduction while achieving competitive structure-recovery performance, with favorable structural Hamming distance across many evaluated settings. These results show that directional evidence can provide an effective preprocessing mechanism for reducing the computational burden of exact DAG learning without requiring a fixed parametric structural~model.}


\maketitle

\section{Introduction}
As a core task in causal inference, causal discovery aims to recover causal relationships among variables by hypothesizing and constructing possible parent-child links leading to a Directed Acyclic Graph (DAG). Such a task has been found particularly helpful for researchers in selecting appropriate confounders, avoiding bias, and designing effective interventions by identifying true causes rather than mere correlations \citep{Cunningham2021, Williamson2021}. This is found in literature across multiple scientific disciplines, e.g., genomics \citep{pingault2018using}, neuroscience \citep{seth2015granger}, economics \citep{henningsen2025estimating}, and the social sciences \citep{cowls2015causation}, among~other~domains.

In the realm of DAG learning, we can categorize learning algorithms into constraint-based, score-based, and hybrid. Constraint-based algorithms use Conditional Independence (CI) tests to infer the graph that underpins causal or associative interactions among variables \citep{spirtes2000causation,spirtes2010introduction}. These CI-test methods are unreliable in finite samples and therefore require massive datasets, which may be unrealistic in many applications \citep{zhang2017causal}. Score-based algorithms, on the other hand, optimize decomposable score functions such as the Bayesian Information Criterion (BIC) or Akaike Information Criterion (AIC) \citep{bouckaert1995bayesian,gamez2011learning}. Among the vast literature of DAG search algorithms, the Integer Programming (IP), as an exact optimization approach, is a widely used score-based method \citep{bartlett2017integer}. Existing IP-based DAG learning methods primarily focus on identifying the best parent set for each node in the network \citep{manzour2021integer}. Starting without any \emph{a priori} knowledge would make the problem explode in the number of variables, and a complete search becomes impossible. The limitations of score- and constraint-based methods motivate hybrid DAG learning approaches that incorporate features from both classes. A representative example is Max-Min Hill-Climbing (MMHC), which first uses a constraint-based procedure to estimate the graph skeleton, then applies a greedy, score-based search to orient edges and refine the structure \citep{tsamardinos2006max}. Related methods, such as H2PC, show that restricting the search space through CI information can substantially improve scalability and statistical efficiency \citep{gasse2015optimality}. 
However, existing hybrid methods primarily rely on dependency information, while directional information is either underutilized or enforced only locally. They rarely exploit directional evidence as global structural constraints within exact optimization.
As a result, hybrid methods may remain unstable when CI tests are ambiguous, sample sizes are limited, or nonlinear dependence is present \citep{scutari2018learns}. 
They may fail to distinguish noisy local dependencies from structurally persistent directional patterns. Incorporating reliable directional evidence as structural constraints can significantly reduce the feasible orientation space while preserving flexibility for uncertain~edge~directions.

To this end, we present a unified hybrid framework for learning DAGs over both discrete and continuous causal networks, which we refer to as DECO (Directional Evidence-guided Configuration Optimization).
Here, we first extract an undirected skeleton from the observational data using a statistical dependency measure; then we use a divergence-based directional measure to infer reliable orientations and restrict infeasible candidate orientations.  
Finally, an existing exact score-based optimizer searches over the resulting restricted parent-set space to recover the highest-scoring~DAG.

Our key contributions can be summarized as follows.
\begin{itemize}

    \item We introduce a non-parametric directional orientation metric based on maximal conditional divergence. 
    Our metric exploits pairwise distributional asymmetries in both continuous and discrete data, effectively accounting for local shifts in location and scale, to propose edge directions.

    \item We propose a method for converting this directional evidence into admissible parent-set restrictions that can be supplied to any exact parent-set optimizer. 
DECO translates divergence scores into targeted restrictions on parent sets. This yields an exponential reduction in the admissible parent-set configuration space while preserving mathematical flexibility when edge orientations are uncertain.

    \item We provide a theoretical characterization of the resulting search-space reduction showing that a linear reduction in the candidate parent-set size per node yields an exponential reduction in the admissible parent-set configuration space, and quantify how edge-level omission risk compounds with node in-degree and graph size in determining retention of the true parent structure

    


\end{itemize}

\section{Methodology}
In this section, we present the general framework of DECO for causal discovery. Let $V$ denote the set of nodes, or vertices, with $\left|V\right|=D$, and $X\in\mathcal{X}^{N\times D}$ be the observed dataset with $N$ samples. Here, $\mathcal{X}$ denotes the variable domain, which may be continuous or discrete. The ultimate goal is to learn the underlying graph structure from $X$. 
Our proposed methodology consists of (i) extracting structural and directional evidence from the observed data and converting them into an admissible parent-child construction, and (ii) applying existing score-based exact optimization over the restricted space to recover the structure of the highest-scoring DAG. 
We will go into detail on step one in Sections 2.1 and 2.2 and on the optimization procedure in Section~2.3.

\subsection{Directional evidence extraction}
The primary objective of the preprocessing phase is to systematically prune the feasible graph search space by eliminating empirically unsupported relations, thereby reducing the computational burden on the downstream global optimization. 
Hybrid methods such as MMHC \citep{tsamardinos2006max} and H2PC \citep{gasse2014hybrid} reduce the optimization space by estimating an undirected skeleton using conditional independence tests, leaving the directional decisions largely to subsequent optimization.
\cite{rolland2022score} and \cite{sanchez2022diffusion} find topological order using Hessian (or its suitable approximation) of the log-likelihood of the data to iteratively identify and remove leaf nodes and use the topological information as a base before seeking causal relations. However, such methods strictly require differentiable likelihood functions and are largely restricted to parametric models. To overcome these limitations, we introduce a non-parametric framework that operates independently of structural functional forms or differentiability assumptions, and it is applicable to both discrete and continuous data distributions. Our pipeline extracts directional evidence by first isolating plausible dependencies via a skeleton construction phase, ranking candidate edge orientations through local dominance, and subsequently validating them via asymmetry-based directional inference.

First, an undirected graph skeleton is established by evaluating the pairwise Mutual Information (MI) across all variables, which quantifies the reduction in uncertainty of one variable given another \citep{cover1991entropy}. Under standard causal Markov and faithfulness assumptions, statistical independence implies the absence of a structural edge \citep{spirtes2000causation, pearl2009causality}. While \cite{chow1968approximating} utilized an MI-maximizing framework to construct tree structures, we apply a conservative, user-defined threshold $\delta_{\text{MI}}$ to eliminate candidate edges with negligible mutual dependence. To prevent structural fragmentation, a local threshold relaxation is dynamically applied to any node isolated at this screening step, thereby connecting it to a small set of candidate neighbors that share the highest mutual information with the isolated node. This guarantees that every variable retains at least a limited set of admissible parent candidates for subsequent optimization.
Because mutual information is inherently symmetric, the resulting skeleton delineates the undirected constraint space containing all plausible causal interactions without determining edge orientations. We conducted a preliminary sensitivity analysis over a range of the threshold, $\delta_{\text{MI}}$, on representative synthetic graphs. The selected thresholds provided a favorable trade-off between preserving true edges and reducing the search space. Details are provided in Appendix~C.1.

Next, we use an asymmetrical measure to obtain the directional preference of the preserved edges. \cite{duong2022bivariate} proposed a metric, Conditional Divergence based Causal Inference (CDCI), for both discrete and continuous data. CDCI estimates the average divergence between a variable's marginal distribution and its standardized conditional distributions, given the other variable, to infer the causal direction between the two variables. This approach removes any strict assumptions on functional forms between variables and their distribution models, and it supports general divergence metrics. However, CDCI is formulated for pairwise causal directions only and thus is not suitable for determining the complete network structure. In addition, by standardizing the conditional distributions, CDCI suppresses location and scale differences between the marginal and conditional distributions that may be informative about causal direction. Moreover, because CDCI averages divergence over all conditioning values, strong directional effects that occur only for a subset of conditions can be diluted by weaker effects elsewhere.
Motivated by CDCI, we propose a pairwise measure, Maximized Conditional Divergence (MCD), which maximizes the divergence between the marginal and conditional empirical distributions for all ordered pairs. If $(X_u,X_v)$ are two variables from the observational data, we define $MCD$ of $X_v$ given~$X_u$~as
\begin{align}
    MCD\left(v|u\right)&{}=\underset{x\in \mathcal{X}_u}{\max} ~{d\left(P\left(X_v|X_u=x\right), P\left(X_v\right)\right)}     \label{M}
\end{align}
where $d$ is any valid probability distance measure, e.g., total variation, Hellinger distance, etc. It is important to note that these distributions are purely empirical, derived solely from observational data. Access to the actual distributions or even its form is not required to find $MCD$.
The intuition is that when conditioned on a true parent, the conditional distributions of $X_v$ will diverge from its marginal distribution. Therefore, if conditioning on $X_u$ significantly alters the distribution of $X_v$, we interpret it as strong directional evidence supporting $X_u\xrightarrow[]{}X_v$ than for the reverse direction.


$MCD$ aims to identify which variable exhibits the strongest distributional shift for any configuration. This measure is inherently asymmetric, as $MCD\left(v|u\right) \neq MCD\left(u|v\right)$ in general. This aligns with a margin-based view, where larger divergence differences provide stronger evidence for causal direction. However, $MCD$ values are not directly comparable across source nodes due to scale differences. As a result, an edge with relatively low global $MCD$ may still be the strongest outgoing candidate for a given source node. To capture such local dominance, we introduce a row-normalized metric $NMCD$, where
\begin{align}
    NMCD\left(v|u\right)&{}=\begin{cases}
        \frac{MCD\left(v|u\right)}{\underset{w}{\max}MCD\left(w|u\right)} & \text{if } MCD\left(v|u\right)>0\\ 0 & \text{otherwise}
    \end{cases}\label{M_n}
\end{align}

\noindent which detects the strongest relative shift or causal direction from each source node. 
This normalization removes scale differences across source variables. It also places the outgoing directional strengths from each source node on a common scale, allowing local directional dominance to be compared independently of the absolute magnitude of the directional shift. Importantly, ${\underset{w}{\max}MCD}\left(w|u\right)=0$ if and only if $MCD\left(w|u\right)=0$ for all $w$ implying that the marginal distribution of $X_w$ is identical to every conditional distribution $P(X_w|X_u=x)$ which ultimately makes $NMCD\left(v|u\right)=0$ by definition. 
To orient an edge, we convert $MCD$ and $NMCD$ into asymmetrical directional preference scores and compare the two directions of each edge of the skeleton.
We define the absolute directional preference ($S$) and normalized directional preference ($NS$) in the following way,
\begin{align}
    S_{u\xrightarrow[]{}v} &{}= MCD\left(v|u\right) - MCD\left(u|v\right) \label{S}\\
    NS_{u\xrightarrow[]{}v} &{}= NMCD\left(v|u\right) - NMCD\left(u|v\right) \label{NS}
\end{align}


For continuous variables, we discretize each variable using a mean/standard-deviation-based binning scheme, following \cite{duong2022bivariate}, so that MCD and NMCD can be computed from empirical marginal and conditional distributions in the same nonparametric manner as for discrete data. For a variable with sample mean $\mu$ and standard deviation $\sigma$, the central range $\left[\mu-k\sigma,\mu+k\sigma\right]$ is partitioned into $4k + 1$ equal-width bins, and the two outermost bin edges are extended so that values beyond the central range are absorbed into the adjacent extreme bin rather than discarded. We set k = 2, yielding 9 bins per continuous variable. Constant or near-constant columns are assigned a single bin. The discretized data are then processed identically to native discrete data for MI-based skeleton construction and MCD/NMCD-based directional evidence extraction. This avoids introducing additional density-estimation assumptions and keeps the directional-evidence stage consistent across data types.



\subsection{Admissible parent-set construction}

The directional evidence extracted in Section 2.1 is now converted into admissible parent-set constraints. Our edge orientation strategy is based on two principles: First, $NMCD$ ranks candidate outgoing edges from each source node based on their relative local dominance. Second, directional asymmetry can be supported when either normalized or absolute asymmetry exhibits a sufficiently strong preference for one direction over the other. Accordingly, we process candidate outgoing edges from each source node by $NMCD$.
Then along the descending order of rank, we evaluate directional asymmetry using $NS$ and $S$ scores for an edge and its reverse direction.
If there is a clear directional asymmetry such that $NS_{u\xrightarrow[]{}v} \geq \delta_{NS}$ or $S_{u\xrightarrow[]{}v} \geq \delta_{S}$, we confirm the direction from $u\xrightarrow[]{}v$ as long as it does not create a cycle with respect to the current directed graph. Here, $\delta_{NS}$ and $\delta_{S}$ are the two thresholds calibrated to score scales and set as fixed constants in an experiment. Empirical analyses over a range of threshold values are provided in Figure~\ref{heatmap}. $NS$ and $S$ capture local and global directional preference, respectively. We use them as complementary measures, allowing an orientation to be confirmed either by strong local dominance or by a sufficiently large absolute asymmetry. 
Each oriented edge is assigned a confidence score, obtained by normalizing the directional-evidence metric ($NS$ or $S$) associated with its orientation, so stronger directional asymmetry yields higher confidence. Confidently oriented edges forbid their reverse directions, and for detached nodes it serves as a tie-breaking ranking term when selecting fallback parents. For retained edges, confident orientations strictly enforce parent-child roles, while ambiguous edges maintain bidirectional candidates, allowing the Integer Programming to orient them flexibly during global~optimization. 


The resulting reduction in the admissible parent-set configuration space is quantified in the following theorem. We provide a sketch proof of the theorem here; a complete proof appears in Appendix A.


\begin{theorem}[Reduction in Parent-Set Configuration Space]Let $V$ be the set of variables with $|V| = D$. If $\Omega_{full}$ denotes the number of parent-set assignments prior to enforcing global acyclicity constraints, 
$\Omega_{\text{full}} = 2^{D(D-1)}$. Under the proposed DECO preprocessing, let $m_v\le D-1$ denote the number of admissible parents of node $v$ and $\bar{m}=\frac{1}{D}\sum_{v\in V}{m_v}$ be the average candidate parent pool size per node. Then the number of admissible parent-set configurations is $\Omega_{DECO}= 2^{\sum_{v \in V} m_v}= 2^{D\bar{m}}$. Finally, if the average admissible parent pool satisfies $\bar{m} \le (1-\alpha)(D-1)$ for $\alpha \in (0, 1]$, $$\frac{\Omega_{\text{DECO}}}{\Omega_{\text{full}}} \le 2^{-\alpha D(D-1)}$$\end{theorem}

\begin{proof}[Proof Sketch]
    In the unconstrained model, each node may choose any subset of the remaining $D-1$ variables, yielding $2^{D-1}$ candidate parent sets. In the DECO framework, the admissible parent pool of node $v$ is $\mathcal{P}_v$ of size $m_v$. Before imposing global acyclicity constraints, the parent-set choices can be counted through the Cartesian product over nodes. Therefore, $$\Omega_{\text{DECO}} = 2^{\sum_{v \in V} m_v}$$ 
Applying the assumed average pruning fraction $\alpha$ gives the following bound on the ratio of the parent-search configurations:$$\frac{\Omega_{\text{DECO}}}{\Omega_{\text{full}}} = \frac{2^{\sum_{v \in V} m_v}}{2^{D(D-1)}} \le 2^{-\alpha D(D-1)}$$ 
Therefore, a linear reduction in the average admissible parent pool induces an exponential reduction in the candidate parent-set assignments considered by the exact~optimizer.
\end{proof}
This theorem bounds the number of admissible parent-set configurations that the optimizer considers. It does not bound the maximum in-degree of any individual node or the number of valid DAGs. However, reducing the search space is meaningful only if feasible solutions of interest remain accessible. The following proposition establishes this property.

\begin{proposition}[True-Parent Retention under Bounded Omission] For every true directed edge $e$, let the probability that DECO preprocessing excludes its true parent relation be at most $\varepsilon$. Then for a node $v$ with true in-degree $d_v$, the probability that all true parent relations of $v$ remain admissible is $$Pr\left(Pa_G\left(v\right)\subseteq \mathcal{P}_v\right)\geq \max{\left(0,1-d_v\varepsilon\right)}$$
Moreover, for a DAG with $\left|E\right|$ true edges, probability that the true DAG remains feasible after the preprocessing is at least $\max{\left(0, 1-\left|E\right|\varepsilon\right)}$.

Under independent edge-retention events, the two probabilities strengthen to $\left(1-\varepsilon\right)^{d_v}$ and $\left(1-\varepsilon\right)^{\left|E\right|}$ respectively.

\end{proposition}

\begin{proof}

    Let $A_e$ denote the event that true parent relation $e$ is omitted during preprocessing. For node $v$, failure to retain all $d_v$ true parents occurs only if at least one of its true parent relations is omitted. By the union bound, $$Pr\left(Pa_G\left(v\right)\nsubseteq \mathcal{P}_v\right)\le \sum_{e\in Pa_G\left(v\right)}Pr\left(A_e\right)\le d_v\varepsilon$$
which gives the bound.

Under no assumption on dependence among events excluding the true parent relations, the union bound gives the stated lower bounds. If the edge-retention events are additionally independent, then $$Pr\left(Pa_G\left(v\right)\subseteq \mathcal{P}_v\right)= \prod_{e\in Pa_G\left(v\right)}Pr\left(A_e^c\right)\geq \left(1-\varepsilon\right)^{d_v}$$
and similarly, probability of $G$ remaining feasible be bounded by $\left(1-\varepsilon\right)^{\left|E\right|}$. 
The independence assumption is not required for the union-bound result and is stated only to obtain the sharper multiplicative bound.
\end{proof}


Algorithm~\ref{algorithm_maxshift} now summarizes the preprocessing procedure that constructs the admissible parent pools used by the exact optimizer.

\begin{algorithm}[ht]
\SetAlgoLined

 \KwIn{Observational dataset, $X\in\mathcal{X}^{N\times D}$} 
 {       

Compute the mutual information matrix MI, directional scores $MCD$ using Equation~\ref{M}, and normalized scores $NMCD$ using Equation~\ref{M_n}\;

Construct the undirected skeleton: $C = \{ \{u,v\} : MI(u,v) \geq \delta_{MI} \}$ \;

Locally relax threshold for isolated node after MI thresholding to obtain fallback candidate neighbors\;

Initialize set of confirmed directed edges, $E_{dir}=\emptyset$ and set of undirected edges, $E_{und}=C$\;

For each source node $u$, rank candidate edges $u\xrightarrow[]{}v$ by $NMCD$\;
  
\For{each candidate $u\xrightarrow[]{}v$ from the ranked order}
  {
  Compute $S$ and $NS$ using Equation~\ref{S} and~\ref{NS} respectively\;
  \uIf{$NS_{u\xrightarrow[]{}v}\geq\delta_{NS} \lor S_{u\xrightarrow[]{}v}\geq\delta_{S}$ \text{and} $u\xrightarrow[]{}v$ does not create a cycle}{add $u\xrightarrow[]{}v$ to $E_{dir}$ and remove $\{u,v\}$ from $E_{und}$\;}
  \lElse{Keep $\{u,v\}$ in $E_{und}$}
  }

  Build the partially directed graph $G' = \left(V, E_{dir} \cup E_{und}\right)$\;
 }

 Construct an admissible parent set $P_v$ for each node $v$ such that

\begin{itemize}
\item confidently directed edges permit only the identified parent direction
\item ambiguous edges remain feasible in both directions.
\end{itemize}

 

 \KwOut{A partially directed graph, $G'$, and admissible parent pool, $\mathcal{P}_v$ for every $v\in V$}
 \caption{Directional-evidence-guided admissible parent-set construction} \label{algorithm_maxshift}
\end{algorithm}

\subsection{Optimization over restricted parent-set space}

The proposed preprocessing is independent of the optimization framework and can be integrated with any exact parent-set optimizer.

In this work, we adopt the classical integer programming formulation for Bayesian network structure learning introduced in \cite{jaakkola2010learning}, further developed via cutting-plane methods by \cite{cussens2011bayesian} and \cite{bartlett2017integer}, as it is the standard benchmark in exact Bayesian network learning. 


Obtained from Algorithm~\ref{algorithm_maxshift}, $\mathcal{P}_v$ is the set of all possible parent nodes of node $v$, and $\mathcal{Q}_v$ denote the power set of $\mathcal{P}_v$. The empty parent set is always included in $\mathcal{Q}_v$, allowing root nodes to be represented within the optimization model. We define the decision variable $z_{v,\mathcal{Q}}$ to denote whether $\mathcal{Q}\in  \mathcal{Q}_v$ is the true set of parent nodes of $v$ ($z_{v,\mathcal{Q}}=1$) or not ($z_{v,\mathcal{Q}}=0$).

The objective function of IP-based DAG learning models is defined in terms of decomposable scores, enabling independent evaluation of parent sets.
Since the preprocessing stage already constrains the set of feasible parent configurations, we use the standard BIC score as the local objective without additional penalty terms.
The optimization model selects $\mathcal{Q}$ for each node $v$. Formally, we define the score of a node–parent set as
\begin{equation}
    Score\left(v,\mathcal{Q}\right)=BIC\left(v,\mathcal{Q}\right) 
    \label{score}
\end{equation}
The resulting integer program maximizes the combinatorial score across all nodes and their parent-configuration search space, subject to the standard convexity and cluster constraints.
This formulation allows the IP to favor structural evidence while permitting corrections when the likelihood strongly supports alternative configurations. The false-positive edges introduced during MI-based skeleton construction can be removed if they yield poor local BN scores, and ambiguous directional decisions can be resolved by global model fit.

We enforce the constraint $\sum_{\mathcal{Q}\in\mathcal{Q}_v}z_{v,\mathcal{Q}}=1$, $\forall v$, to ensure each node has only one subset of parents selected. A valid DAG should not contain any cycles; however, the initial graph obtained from the IP does not guarantee this requirement. If the initial graph does contain cycles, we eliminate them by adding cycle-elimination constraints
\begin{equation}
    \sum_{v\in C}{\sum_{\substack{
\mathcal{Q}\in\mathcal Q_v\\
\mathcal{Q}\cap C=\emptyset
}} z_{v,\mathcal{Q}}} \geq 1
\label{cycle}\end{equation}
and iteratively resolve the IP until a valid DAG is obtained \citep{cussens2011bayesian}. The IP formulation returns one globally optimal DAG within the admissible parent space. As with other score-based approaches, multiple DAGs may achieve the same optimal score; in that case, the solver returns a single optimal representative. The complete MILP formulation over the DECO-restricted parent sets is provided in Appendix~B.

Figure~\ref{fig:steps} illustrates the complete progression of the proposed DECO framework on the Asia benchmark network described in \cite{lauritzen1988local}. The preprocessing stage first constructs an MI-based dependency skeleton, then estimates directional evidence using the proposed asymmetry measures. The resulting admissible directional relations are then incorporated into the constrained optimization stage to obtain the final~DAG.

\begin{figure}[t]
    \centering
    \includegraphics[width=1\linewidth]{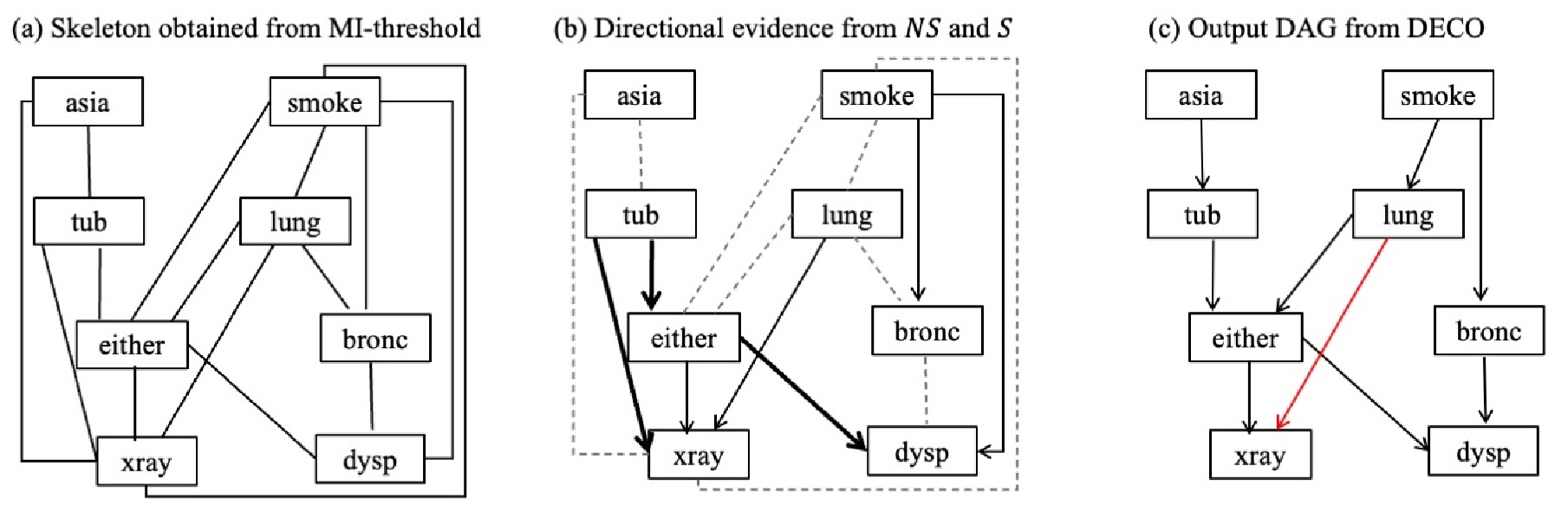}
    \caption{DECO pipeline on the Asia network: (a) Dependency screening from mutual information, (b) Directional evidence extraction and admissible parent-set construction from the proposed asymmetry measures (confidence of these directions is implied through edge thickness, and ambiguous edges are dashed in gray), (c) Final DAG obtained after the exact optimization over the restricted search space (the false positive edge is marked~red).}
    \label{fig:steps}
\end{figure}

Since DECO restricts the admissible parent sets without altering the underlying optimization model, the validity of the cluster constraints and finite termination of the optimization procedure continue to hold. 

\section{Experimental results}

In this section, we present the experimental setup and evaluate DECO's performance at each stage. 
Our evaluation follows the three stages of the proposed framework. We assess the quality of the preprocessing stage by examining (i) skeleton construction accuracy and (ii) the resulting reduction in admissible parent-set search space. Finally, we evaluate the end-to-end DAG recovery performance of DECO on benchmark Bayesian networks from Bayesian Network Repository \citep{scutari2010learning} and synthetic Erd\H{o}s-R\'{e}nyi DAGs. Erd\H{o}s-R\'{e}nyi is a random directed graph parameterized by the number of ordered nodes and probability they are connected with strictly from lower-indexed nodes to higher-indexed nodes to enforce acyclicity \citep{erdds1959random}.  Additionally, we examine the computational cost of the proposed framework and discuss the trade-off between search-space reduction, recovery accuracy, and runtime.

\subsection{Skeleton construction}
To assess the quality of the preprocessing stage, we first evaluate the dependency-screening stage by comparing the recovered skeleton against the true graph skeleton. Skeleton performance is reported using the missing-edge rate (omission) and the extra-edge rate (commission). It is important to choose our priority here wisely. Since the admissible parent sets are constructed directly from the recovered skeleton, a low commission rate will help reduce the search space, but preserving true dependencies at this stage is critical, as omitting true parent relations may exclude the true DAG from the feasible optimization region, as formalized in Proposition~1.
Table~\ref{tab:skeleton} compares the proposed mutual-information (MI) skeleton construction with Maximum Mean Discrepancy (MMD) based dependency screening \citep{gretton2012kernel} and graphical Lasso or GLasso \citep{friedman2008sparse}.

\begin{table}[ht!]    
    \centering
    \caption{Comparison of skeleton quality measured by omission (missing-edge)$=\frac{FN}{TP+FN}$ and commission (extra-edge)$=\frac{FP}{FP+TN}$ rates. Lower omission preserves true parent relations, whereas lower commission yields a smaller admissible search space.}
    \begin{tabular}{|c|c|c|c|}
    \hline
      Network & Method & Omission Rate $\downarrow$ & Commission Rate $\downarrow$\\
      \hline
      \multirow{3}{*}{Asia} & MI & $\mathbf{0.09\pm 0.076}$ & $0.21\pm 0.048$\\
      & MMD & $0.75\pm 0$ & $0.56\pm 0.029$\\
      & G-Lasso & $0.36\pm 0.1020$ & $0.17\pm 0.086$\\
      \hline
      \multirow{3}{*}{Sachs} & MI & $\mathbf{0.01\pm 0.0226}$ & $0.32\pm 
0.008$\\
      & MMD & $0.25\pm 0.0318$	& $0.68\pm 0.017$\\
      & G-Lasso & $0.56\pm 0.0591$ & $0.05\pm 0.022$\\
      \hline
      \multirow{3}{*}{Child} & MI & $\mathbf{0.17\pm 0.0438}$ & $0.14\pm 
0.013$	\\
      & MMD & $0.32\pm0$	& $0.73\pm 0.009$ \\
      & G-Lasso & $0.75\pm 0.02$ & $0.04\pm 0.005$\\
      \hline
    \multirow{3}{*}{Insurance} & MI & $\mathbf{0.34\pm0.0147}$ & $0.21\pm 0.030$\\
    & MMD & $0.41\pm 0.0149$ & $0.80\pm 0.007$\\
    & G-Lasso & - & - \\
    \hline
    \multirow{3}{*}{Water} & MI & $0.63\pm 0.139$	& $0.10\pm 0.004$\\
    & MMD & $\mathbf{0.43\pm 0.0237}$ & $0.75\pm 0.0107$\\
    & G-Lasso & - & - \\
    \hline
    \multirow{3}{*}{Barley} & MI & $0.35\pm 0.0128$ &	$0.17\pm 0.007$\\
    & MMD & $\mathbf{0.32\pm 0.0104}$ & $0.80\pm 0.008$\\
    & G-Lasso & - & - \\
    \hline

    \end{tabular}
    
    \label{tab:skeleton}
\end{table}

Table 1 shows that MI generally provides a favorable balance between omission and commission, with particularly low omission on several benchmark networks while avoiding the very high commission rates observed for MMD.
Since omission errors may eliminate true parent relations and consequently exclude the true DAG from the feasible optimization space (Proposition 1), DECO prioritizes preserving true dependencies even at the cost of retaining a small number of additional candidate edges. Despite this conservative strategy, the preprocessing still substantially reduces the search space, as demonstrated in the following analysis. 

On MI-skeleton, we examine the trade-off between preserving true dependencies and reducing the admissible search space with different values of $\delta_{MI}$. Figure~\ref{trade} shows the results of Erd\H{o}s-R\'{e}nyi DAGs with fixed size of $n=50$ and edge probability $p= 0.05$ while we vary the number of samples, $N$. From Figure~\ref{trade}, one can see that for omission rates up to approximately $0.25$, increasing the threshold produces substantial additional search-space reduction; beyond this region, the marginal gain in reduction becomes small relative to the increase in omission.
Since we want the search space reduction as high as possible without jeopardizing the type II error, we fix $\delta_{MI}=0.10$ for this ER-setting, and it is the same threshold used to produce the results from Table~\ref{tab:skeleton}. 
The qualitative trade-off remained consistent across the evaluated settings, although the preferred operating region varied with graph size and sample availability; additional results are provided in Appendix~C.1.

\begin{figure}[t]
    \centering
    \includegraphics[width=1\linewidth]{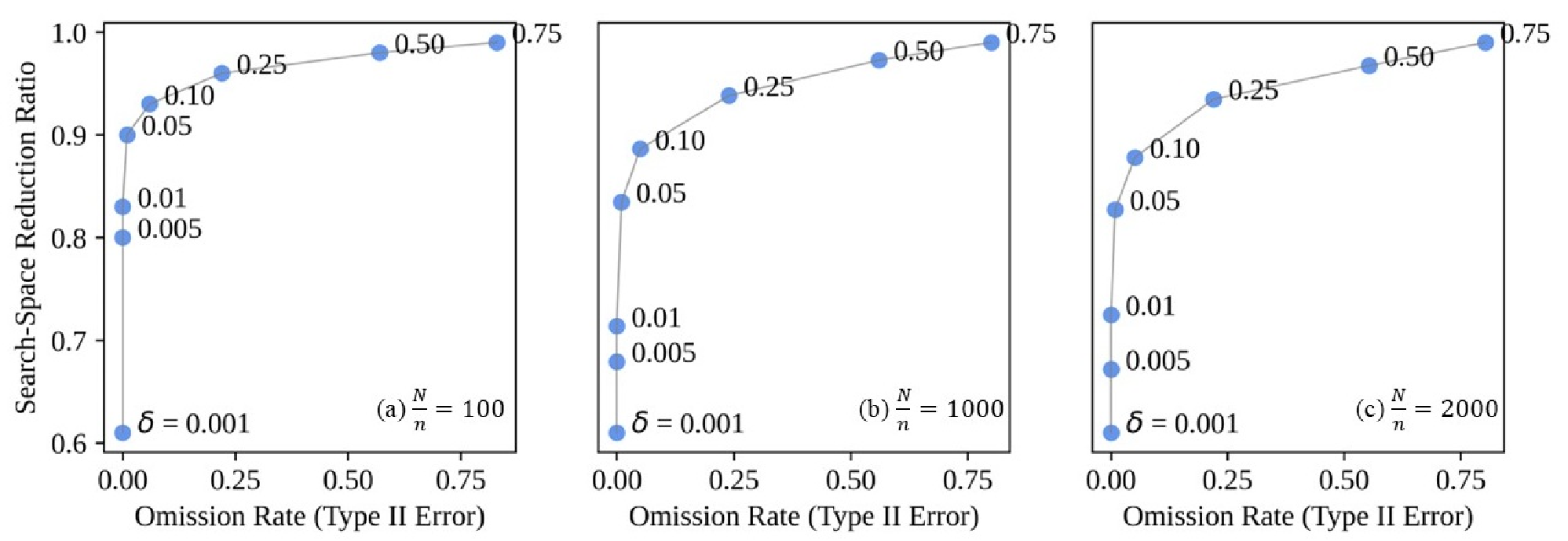}
    \caption{Trade-off between search-space reduction and Type II error illustrating the effect of increasing the MI threshold. Each point corresponds to one threshold value for the ER graph with $n=50, p = 0.05$ and variable $N$.}
    \label{trade}
\end{figure}

To further evaluate the effectiveness of the proposed preprocessing stage, we measure the fraction of admissible parent-set configurations eliminated prior to optimization. Figure~\ref{fig:space reduction} reports the average reduction in search space over $50$ random Erd\H{o}s-R\'{e}nyi DAG realizations of varying graph sizes. 
Results are shown separately for discrete and continuous synthetic data. DECO consistently eliminates a substantial fraction of candidate parent-set configurations, reducing the search space by approximately $50-60\%$ on discrete datasets and $30-55\%$ on continuous datasets. These observations empirically support the theoretical results on search-space reduction. Together, Table~1, Figure~\ref{trade} and Figure~\ref{fig:space reduction} demonstrate that the preprocessing stage achieves a favorable balance between preserving true dependencies and reducing the combinatorial search space.


\begin{figure}[t]
    \centering
    \includegraphics[width=0.5\linewidth]{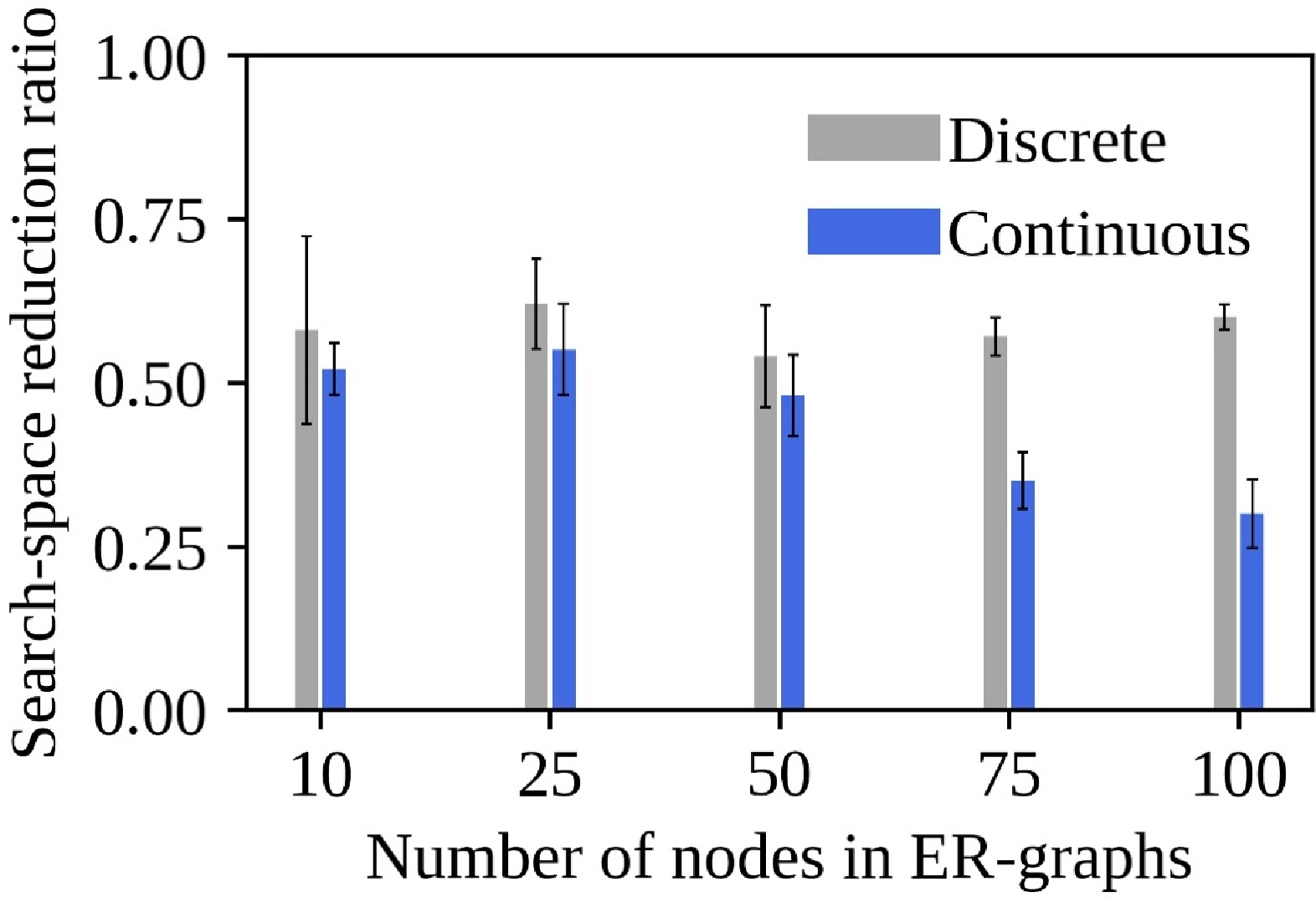}
    \caption{Average search-space reduction achieved by DECO on Erd\H{o}s-R\'{e}nyi DAGs of varying size and edge probability $=0.15$. The reduction ratio is computed as $1-\frac{\Omega_{DECO}}{\Omega_{Full}}$}
    \label{fig:space reduction}
\end{figure}


\subsection{Edge orientation}
Thresholds $\delta_S$ and $\delta_{NS}$ jointly control the degree of confidence required before DECO commits to an edge orientation. Figure~\ref{heatmap} examines this trade-off for the ER-graph with $n=50, p=0.05$, and $N=5,000$. We demonstrate the sensitivity in terms of (a) correct-orientation coverage, measured as the fraction of true edges that are correctly oriented; (b) ambiguity rate, measured as the fraction of true edges left unoriented; and (c) orientation precision, measured as the fraction of committed orientations that are correct.

\begin{figure}[t]
    \centering
    \includegraphics[width=1\linewidth]{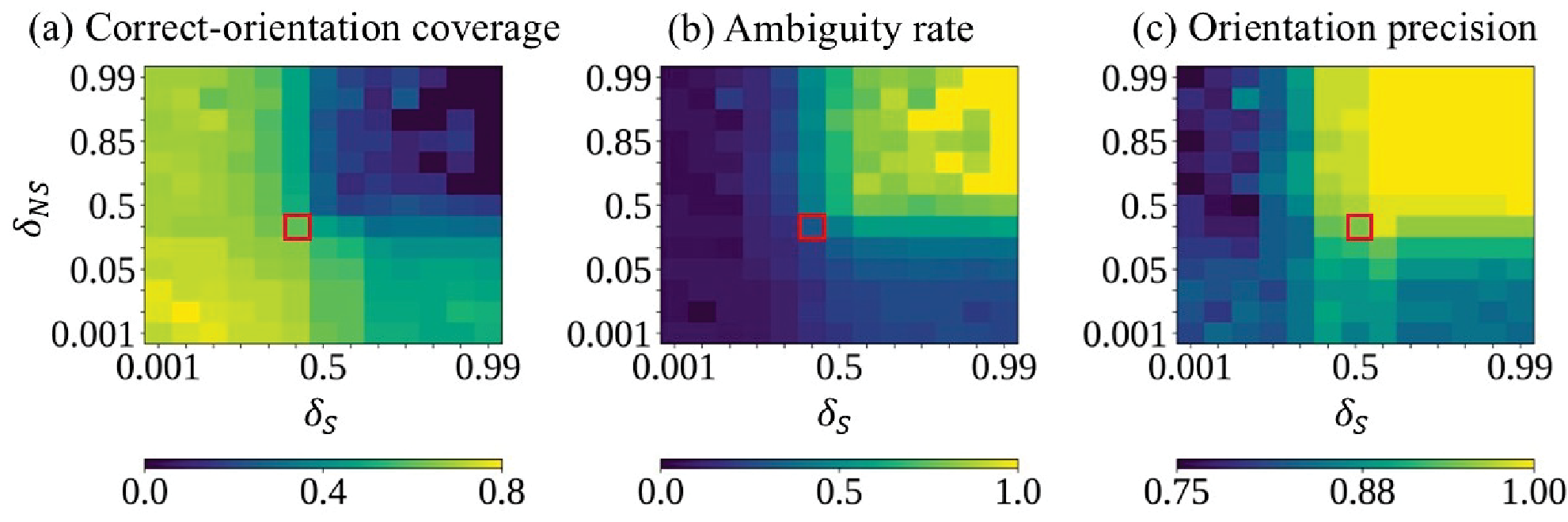}
    \caption{Sensitivity of directional evidence thresholds$\delta_S$ and $\delta_{NS}$. Increasing the thresholds generally improves the reliability of committed orientations at the expense of directional coverage and increased ambiguity. The outlined cell indicates the threshold pair used in the main experiments.}
    \label{heatmap}
\end{figure}

Increasing either threshold generally improves the reliability of committed orientations but reduces directional coverage and increases the number of ambiguous edges. The selected threshold pair lies in an intermediate regime of $\left(\delta_S,\delta_{NS}\right)=\left(0.25, 0.5\right)$ that achieves high orientation precision while retaining sufficient directional information to meaningfully restrict the subsequent optimization.

\subsection{End-to-end DAG recovery}
For end-to-end DAG recovery, we evaluate DECO on both benchmark Bayesian networks and synthetic Erd\H{o}s-R\'{e}nyi DAGs (see Table~2). For the benchmark networks, observation datasets are generated from the corresponding ground-truth Bayesian networks using forward sampling. 
For the synthetic Erd\H{o}s-R\'{e}nyi directed acyclic graphs (ER-DAGs), we sample a random topological ordering of $n$ nodes and add directed edges from earlier to later nodes with probability $p$ for each experiment. This topological sampling ensures that the generated graph is acyclic. Then, discrete datasets are generated using conditional probability tables associated with each node, while continuous datasets are generated from linear or non-linear structural equation models with additive Gaussian noise. 

For each network, we conducted the experiments with sample sizes of $1,000$, $5,000$, and $10,000$, and repeated each experiment over 50 independently generated datasets. 
We compare DECO against PC, GES, and Hill-Climbing (HC) on discrete datasets. For continuous synthetic datasets, we additionally compare against NOTEARS and GOLEM. We implemented PC and GES using \texttt{Causal-Learn} \citep{zheng2024causal} and HC using \texttt{pgmpy} packages \citep{Ankan2024pgmpy} on Python. GOLEM and NOTEARS were implemented using  \texttt{gCastle} package \citep{zhang2021gcastle}. GOBNILP \citep{cussens2020gobnilp} was run using pygobnilp [version 1.0], with the maximum parent-set size of 3, and pygobnilp's default score-dominance pruning under which a candidate parent set is retained only if its local score exceeds that of all its proper subsets. GOBNILP's IP was solved using Gurobi. This configuration constitutes a standard and non-directional pruning baseline against which DECO's performance can be directly compared.
The integer program in DECO was implemented using the PuLP Python library \citep{Mitchell2011PuLPA} and solved with PuLP's default backend, the open-source CBC (COIN-OR Branch and Cut) solver \citep{cbcforrest}, with cluster constraints added iteratively as cycles were detected.

The representative DAGs returned by the implementations were used directly for the metric evaluations. Performance is assessed using Structural Hamming Distance (SHD), False Discovery Rate (FDR), True Positive Rate (TPR), False Positive Rate (FPR), and runtime, and compared against PC, GES, HC, GOBNILP, and continuous-data baselines where applicable.


\begin{table}[ht!]
    \centering
    \caption{Reference networks with their features}
    \begin{tabular}{|c|c|c|c|}
     
    \hline
       Network & Type & Node & Arc/Density  \\
       \hline
       Asia  & BN &  8 & 8  \\
       Sachs  & BN & 11 & 17  \\
       Child &  BN &20 & 25  \\
       Insurance & BN & 27 & 52 \\
       Water & BN & 32 & 66  \\
       Hailfinder & BN & 56 & 66 \\
       ER-25 & Discrete & 25 & 0.15, 0.25\\
       ER-50 & Discrete & 50 & 0.05\\
       ER-25 & Non-linear Gaussian & 25 & 0.05, 0.15\\
       ER-50 & Non-linear Gaussian & 50 & 0.05\\
       ER-25 &  Linear Gaussian & 25 & 0.15\\
       \hline
    \end{tabular}
   
    \label{tab:placeholder}
\end{table}


Figure~\ref{fig:box} compares the SHD of DECO and competing methods on representative datasets. Since the observed SHD values span over a large range across datasets, the y-axis is displayed on a logarithmic scale to facilitate visual comparison. On the benchmark discrete networks (Figure~\ref{fig:box}(a)), 
DECO achieves among the lowest SHD on Asia, Child, Insurance, and Water, while GOBNILP obtains lower SHD on Sachs. 
On the larger Hailfinder network, HC achieves lower SHD, whereas DECO remains substantially better than PC and GES. 
This behavior suggests that the proposed admissible-parent construction is highly effective for structure recovery on small and medium-sized networks, although aggressive pruning may occasionally exclude parent configurations that are beneficial for recovering certain large-scale structures.


\begin{figure} [t!]
    \centering
    \includegraphics[width=0.93\linewidth]{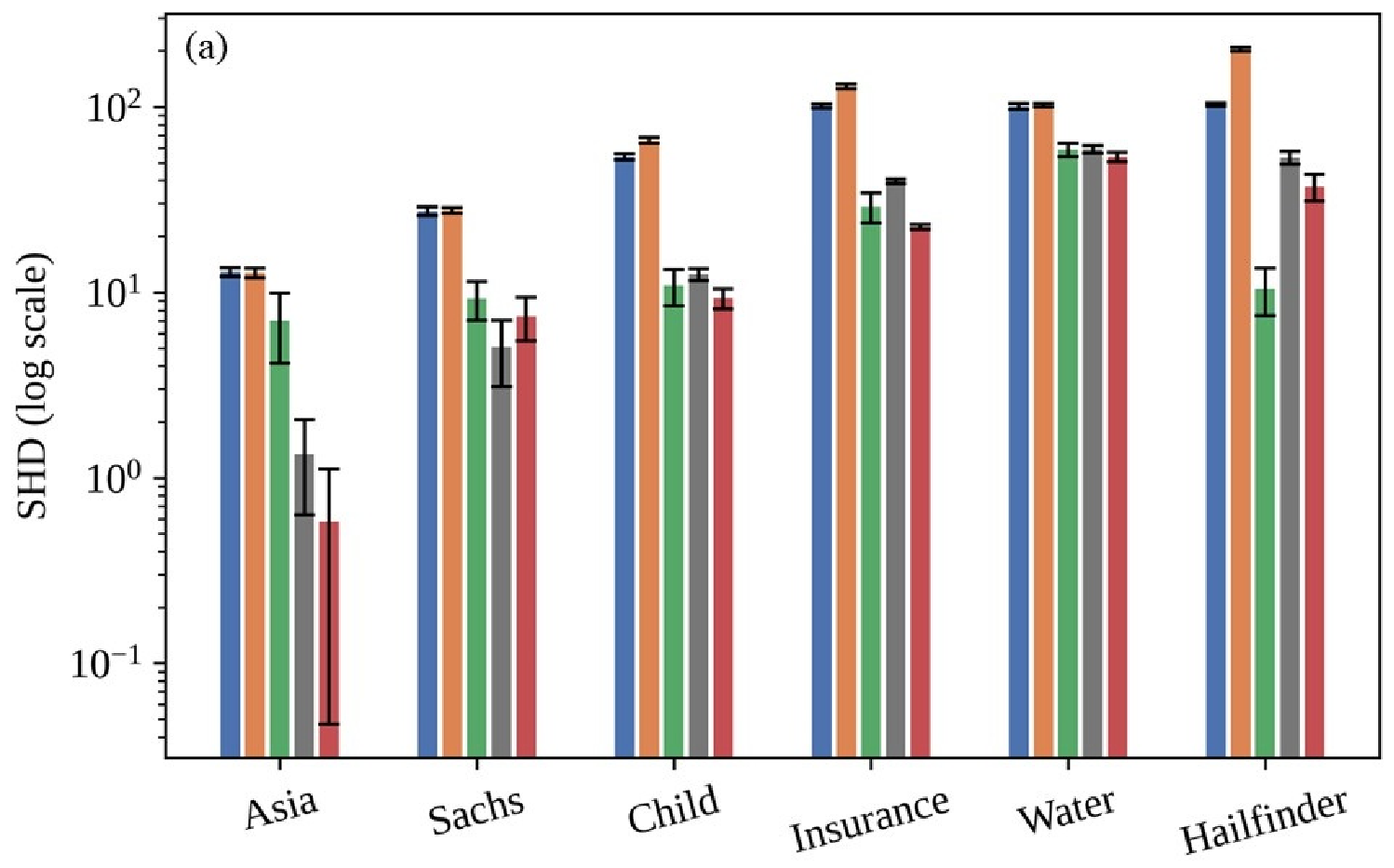}
    \includegraphics[width=0.87\linewidth]{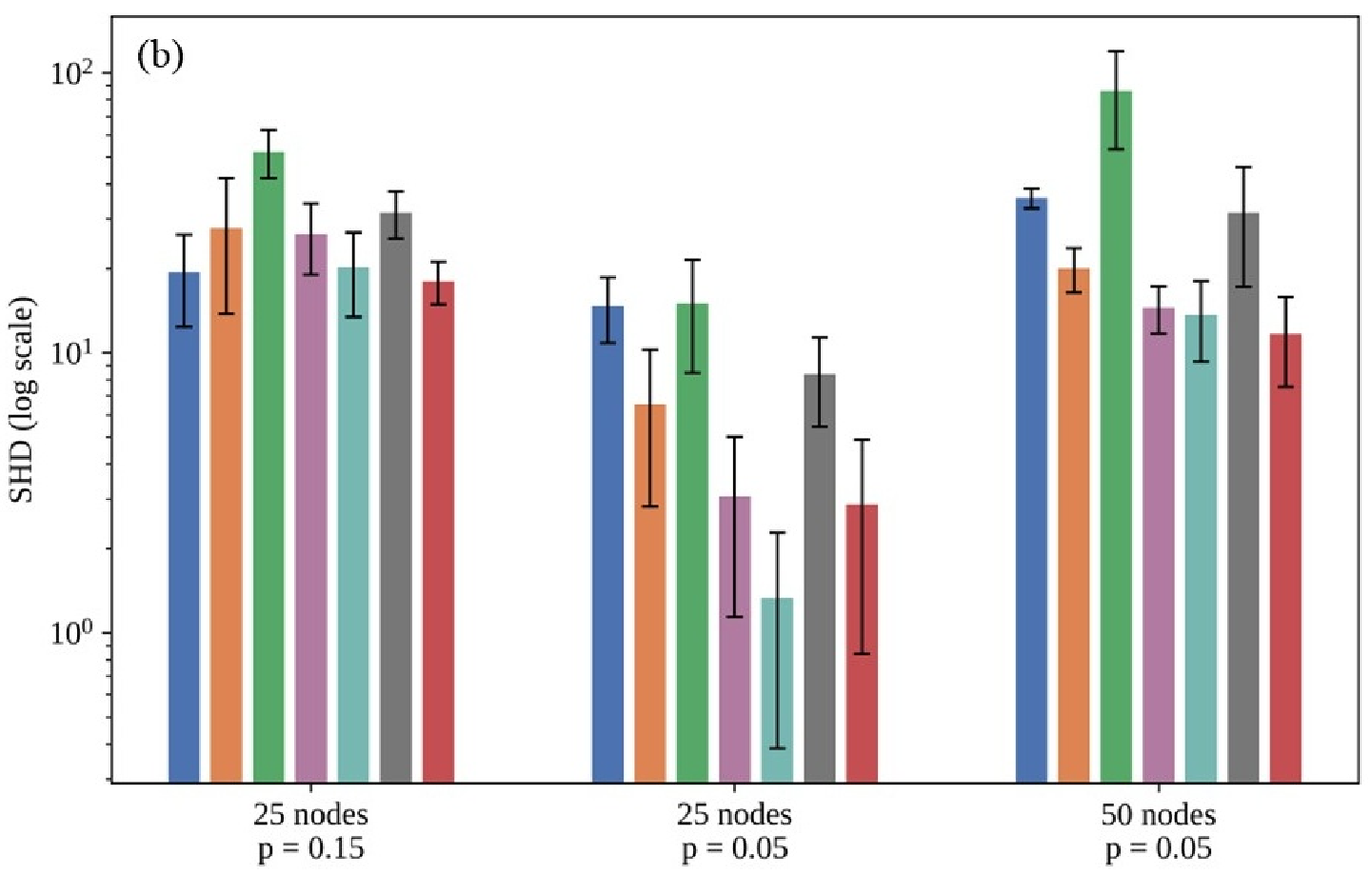}
    \includegraphics[width=1\linewidth]{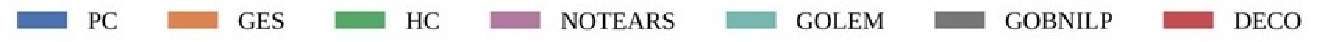}
    \caption{Structural Hamming Distance (SHD) comparison of DECO and other competing methods on (a) benchmark discrete Bayesian networks and (b) nonlinear Gaussian Erd\H{o}s-R\'{e}nyi DAGs, all at $N=10,000$ observations. Error bars denote one standard deviation over repeated runs. Lower values indicate better structure recovery. The y-axis is displayed on a logarithmic scale.}
    \label{fig:box}
\end{figure}

On nonlinear continuous Erd\H{o}s-R\'{e}nyi graphs (Figure~\ref{fig:box}(b)), DECO achieves the lowest SHD in two of the three evaluated settings and remains competitive with NOTEARS and GOLEM in the remaining case. In particular, DECO outperforms all competing methods on the 25-node ($p=0.15$) and 50-node ($p=0.05$) graphs, while GOLEM achieves lower SHD on the sparser 25-node ($p=0.05$) setting.
These results suggest that the proposed directional-evidence constraints can preserve sufficient flexibility for accurate graph recovery while substantially reducing the search space.
Complete numerical results, including FDR, TPR, and FPR, are provided in Appendix~C.2. SHD is used as the primary metric in Figure~\ref{fig:box} because it jointly reflects both missing and spurious edges.

While DECO achieves favorable SHD across the nonlinear continuous settings, this advantage does not translate into uniform dominance across individual edge-recovery metrics. As shown in Appendix~C.2, NOTEARS and GOLEM frequently attain lower FDR and FPR, whereas DECO recovers a larger fraction of true edges in several settings, as reflected by its comparatively higher TPR. For example, for the nonlinear $n=25, p=0.15, N=10,000$ setting, DECO achieves a TPR of $0.68$, compared with $0.55$ for GOLEM and $0.40$ for NOTEARS, while GOLEM and NOTEARS obtain lower FDR. This indicates a precision–recall trade-off: DECO retains greater structural coverage, which can recover additional true relations but may also admit more false-positive relations. Consequently, favorable aggregate SHD does not necessarily imply simultaneous superiority in FDR, TPR, and FPR.

In the linear Gaussian setting, this distinction becomes more pronounced. Both GOLEM and NOTEARS achieve lower SHD than DECO, consistent with their design for continuous structural equation models under assumptions closely aligned with this data-generating setting. Detailed results are provided in Appendix~C.2.


Across most benchmark and synthetic datasets, increasing the number of observations improves DECO's recovery performance. Larger sample sizes yield more reliable estimates of mutual information and directional asymmetry, enabling more accurate admissible-parent construction and improved optimization outcomes. Consequently, SHD and FDR generally decrease while TPR tends to improve as the sample size increases. Although DECO frequently achieves the lowest SHD and FDR, its TPR is not always the highest on larger graphs. This reflects the conservative nature of the preprocessing stage: by aggressively removing implausible parent relations, DECO reduces false discoveries but may occasionally exclude true dependencies, leading to slightly lower recall.

Furthermore, we monitored the computational cost of the optimization stage. On nonlinear Gaussian ER graphs with 50 nodes and 10,000 observations, DECO required approximately $146\pm17$ seconds, comparable to HC ($139\pm206$ seconds) and substantially lower than GOLEM ($1,705\pm100$ seconds). GOBNILP achieved the lowest runtime among the exact optimization approaches in this comparison ($10\pm2$ seconds). This difference should be interpreted in light of the optimization implementations: GOBNILP employs the highly optimized Gurobi solver together with a fixed maximum parent-set size of three and built-in score-dominance pruning, whereas DECO uses the general-purpose open-source CBC solver without imposing a fixed maximum in-degree. Thus, the runtime comparison reflects not only the parent-set reduction strategies but also substantial differences in solver implementation and pruning mechanisms.

The objective of DECO is therefore not to outperform specialized exact Bayesian-network solvers in raw optimization time, but to investigate whether data-driven dependency and directional evidence can reduce the admissible parent-set space while retaining accurate structural recovery. In the evaluated setting, DECO remains computationally practical despite using a general-purpose open-source MILP solver, while often providing favorable structural recovery relative to the competing methods. For example, on the 50-node nonlinear Gaussian ER graphs considered here, DECO achieves lower SHD than GOBNILP despite its higher computational cost. These results highlight a trade-off between computational efficiency and structural recovery and suggest that further improvements in the downstream optimization implementation could complement DECO's preprocessing-based search-space reduction.

\section{Conclusion}

In this paper, we propose DECO, a 
directional evidence-guided framework for exact DAG learning. The preprocessing combines dependency screening and distributional directional evidence to construct admissible parent sets before applying a classical exact optimization framework.
Theoretical analysis showed that the proposed preprocessing stage reduces the admissible parent-set search space exponentially relative to the unconstrained formulation. Experimental results on benchmark Bayesian networks and simulated DAGs demonstrated that DECO substantially reduces the search space while achieving competitive or superior structure recovery performance across a variety of discrete and nonlinear continuous settings. Detailed edge-recovery metrics further reveal a precision–recall trade-off, particularly for continuous data, where greater true-edge coverage can be accompanied by higher false-discovery rates than more conservative competing methods.

One of the main limitations of the current framework is that the directional evidence is derived from pairwise distributional asymmetry and may not fully capture higher-order causal interactions among variables. Consequently, directional decisions currently rely on confidence thresholds that may discard useful information in more complex dependency structures. Second, although the proposed preprocessing significantly reduces the search space, the exact optimization stage can remain computationally demanding for larger graphs. In addition, extending the framework to mixed-type variables and more general structural equation models remains an important direction for future research.

\section*{Statements and Declarations}

\subsection*{Data availability}
The benchmark Bayesian network structures and parameters used in this study (Asia, Sachs, Child, Insurance, Water, Barley, and Hailfinder) were obtained from the bnlearn Bayesian Network Repository, maintained by Marco Scutari, and are publicly available at https://www.bnlearn.com/bnrepository/. The observational datasets used in the benchmark experiments were generated from these networks by forward sampling as described in the manuscript. The synthetic Erdős–Rényi DAGs and corresponding discrete and continuous datasets were generated using the procedures described in Section 3.3. No restricted-access or proprietary datasets were used.

\bibliography{sn-bibliography}

@inproceedings{scutari2018learns,
  title={Who learns better bayesian network structures: Constraint-based, score-based or hybrid algorithms?},
  author={Scutari, Marco and Graafland, Catharina Elisabeth and Guti{\'e}rrez, Jos{\'e} Manuel},
  booktitle={International conference on probabilistic graphical models},
  pages={416--427},
  year={2018},
  organization={PMLR}
}

@article{erdds1959random,
  title={On random graphs I},
  author={Erd{\H{o}}s, P{\'a}l and R{\'e}nyi, Alfr{\'e}d},
  journal={Publ. math. debrecen},
  volume={6},
  number={290-297},
  pages={18},
  year={1959}
}

@article{zheng2024causal,
  title={Causal-learn: Causal discovery in python},
  author={Zheng, Yujia and Huang, Biwei and Chen, Wei and Ramsey, Joseph and Gong, Mingming and Cai, Ruichu and Shimizu, Shohei and Spirtes, Peter and Zhang, Kun},
  journal={Journal of Machine Learning Research},
  volume={25},
  number={60},
  pages={1--8},
  year={2024}
}

@article{gretton2012kernel,
  title={A kernel two-sample test},
  author={Gretton, Arthur and Borgwardt, Karsten M and Rasch, Malte J and Sch{\"o}lkopf, Bernhard and Smola, Alexander},
  journal={The journal of machine learning research},
  volume={13},
  number={1},
  pages={723--773},
  year={2012},
  publisher={JMLR. org}
}

@inproceedings{zhang2017causal,
  title={Causal discovery using regression-based conditional independence tests},
  author={Zhang, Hao and Zhou, Shuigeng and Zhang, Kun and Guan, Jihong},
  booktitle={Proceedings of the AAAI conference on artificial intelligence},
  volume={31},
  number={1},
  year={2017}
}

@inproceedings{gasse2015optimality,
  title={On the optimality of multi-label classification under subset zero-one loss for distributions satisfying the composition property},
  author={Gasse, Maxime and Aussem, Alexandre and Elghazel, Haytham},
  booktitle={International Conference on Machine Learning},
  pages={2531--2539},
  year={2015},
  organization={PMLR}
}

@article{tsamardinos2006max,
  title={The {M}ax-{M}in {H}ill-{C}limbing {B}ayesian network structure learning algorithm},
  author={Tsamardinos, Ioannis and Brown, Laura E and Aliferis, Constantin F},
  journal={Machine learning},
  volume={65},
  number={1},
  pages={31--78},
  year={2006},
  publisher={Springer}
}

@article{cussens2011bayesian,
  author = {Cussens, James},
  title = {Bayesian network learning with cutting planes},
  journal = {UAI 2011, Proceedings of the 27th Conference on Uncertainty in Artificial Intelligence},
  year = {2011},
  url = {https://arxiv.org/pdf/2304.02146}
}

@article{scutari2010learning,
  author = {Scutari, Marco},
  title = {Learning Bayesian Networks with the bnlearn R Package},
  journal = {Journal of Statistical Software},
  volume = {35},
  number = {3},
  pages = {1--22},
  year = {2010},
  url = {https://www.jstatsoft.org/article/view/v035i03/394}
}

@article{gamez2011learning,
  author = {Gámez, José A. and Mateo, Juan L. and Puerta, José M.},
  title = {Learning {B}ayesian networks by hill climbing: efficient methods based on progressive restriction of the neighborhood},
  journal = {Data Mining and Knowledge Discovery},
  volume = {22},
  pages = {106--148},
  year = {2011},
  doi = {10.1007/s10618-010-0178-6}
}

@book{Cunningham2021,
  author = {Cunningham, Scott},
  title = {Causal Inference: The Mixtape},
  publisher = {Yale University Press},
  address = {New Haven},
  year = {2021},
  isbn = {9780300251685},
  url = {https://scunning.com}
}

@inproceedings{cussens2020gobnilp,
  title={GOBNILP: Learning {B}ayesian network structure with integer programming},
  author={Cussens, James},
  booktitle={International Conference on Probabilistic Graphical Models},
  pages={605--608},
  year={2020},
  organization={PMLR}
}

@inproceedings{Mitchell2011PuLPA,
  title        = {Pu{LP}: A Linear Programming Toolkit for Python},
  author       = {Stuart Mitchell and Michael O'Sullivan and Iain Dunning},
  year         = {2011},
  url          = {https://optimization-online.org}
}

@incollection{cbcforrest,
  author    = {John Forrest and Robin Lougee-Heimer},
  title     = {{CBC} User Guide},
  booktitle = {Emerging Theory, Methods, and Applications},
  pages     = {257--277},
  year      = {2005},
  publisher = {INFORMS},
  doi       = {10.1287/educ.1053.0020},
  url       = {https://doi.org/10.1287/educ.1053.0020}
}

@article{cover1991entropy,
  title={Entropy, relative entropy and mutual information},
  author={Cover, Thomas M and Thomas, Joy A and others},
  journal={Elements of information theory},
  volume={2},
  number={1},
  pages={12--13},
  year={1991}
}

@article{gasse2014hybrid,
  title={A hybrid algorithm for Bayesian network structure learning with application to multi-label learning},
  author={Gasse, Maxime and Aussem, Alex and Elghazel, Haytham},
  journal={Expert Systems with Applications},
  volume={41},
  number={15},
  pages={6755--6772},
  year={2014},
  publisher={Elsevier}
}

@article{Williamson2021,
  author = {Williamson, Elizabeth J. and Forbes, Andrew},
  title = {A tutorial on directed acyclic graphs},
  journal = {Journal of Clinical Epidemiology},
  volume = {139},
  pages = {257--265},
  year = {2021},
  doi = {10.1016/j.jclinepi.2021.08.012},
  url = {https://pmc.ncbi.nlm.nih.gov/articles/PMC8821727/}
}

@article{friedman2008sparse,
  title={Sparse inverse covariance estimation with the graphical lasso},
  author={Friedman, Jerome and Hastie, Trevor and Tibshirani, Robert},
  journal={Biostatistics},
  volume={9},
  number={3},
  pages={432--441},
  year={2008},
  publisher={Oxford University Press}
}

@article{zhang2021gcastle,
  title={g{C}astle: A Python Toolbox for Causal Discovery},
  author={Zhang, Keli and Zhu, Shengyu and Kalander, Marcus and Ng, Ignavier and Ye, Junjian and Chen, Zhitang and Pan, Lujia},
  journal={arXiv preprint arXiv:2111.15155},
  year={2021}
}

@article{Ankan2024pgmpy,
  title   = {pgmpy: A Python Toolkit for Bayesian Networks},
  author  = {Ankur Ankan and Johannes Textor},
  journal = {Journal of Machine Learning Research},
  year    = {2024},
  volume  = {25},
  number  = {265},
  pages   = {1--8},
  url     = {https://www.jmlr.org/papers/v25/23-0487.html}
}

@article{pingault2018using,
  title={Using genetic data to strengthen causal inference in observational research},
  author={Pingault, Jean-Baptiste and O’reilly, Paul F and Schoeler, Tabea and Ploubidis, George B and Rijsdijk, Fr{\"u}hling and Dudbridge, Frank},
  journal={Nature Reviews Genetics},
  volume={19},
  number={9},
  pages={566--580},
  year={2018},
  publisher={Nature Publishing Group UK London}
}

@article{seth2015granger,
  title={Granger causality analysis in neuroscience and neuroimaging},
  author={Seth, Anil K and Barrett, Adam B and Barnett, Lionel},
  journal={Journal of Neuroscience},
  volume={35},
  number={8},
  pages={3293--3297},
  year={2015},
  publisher={Society for Neuroscience}
}

@article{henningsen2025estimating,
  title={Estimating causal effects with observational data: guidelines for agricultural and applied economists},
  author={Henningsen, Arne and Low, Guy and Wuepper, David and Dalhaus, Tobias and Storm, Hugo and Belay, Dagim and Hirsch, Stefan},
  journal={Journal of Agricultural Economics},
  year={2025},
  publisher={Wiley Online Library}
}

@book{spirtes2000causation,
  title={Causation, prediction, and search},
  author={Spirtes, Peter and Glymour, Clark N and Scheines, Richard},
  year={2000},
  publisher={MIT press}
}

@article{bartlett2017integer,
  title={Integer {L}inear {P}rogramming for the {B}ayesian network structure learning problem},
  author={Bartlett, Mark and Cussens, James},
  journal={Artificial Intelligence},
  volume={244},
  pages={258--271},
  year={2017},
  publisher={Elsevier}
}

@phdthesis{bouckaert1995bayesian,
  title={Bayesian belief networks: from construction to inference},
  author={Bouckaert, Remco Ronaldus},
  year={1995},
  school={Utrecht University}
}

@article{spirtes2010introduction,
  title={Introduction to causal inference.},
  author={Spirtes, Peter},
  journal={Journal of Machine Learning Research},
  volume={11},
  number={5},
  year={2010}
}

@article{manzour2021integer,
  title={Integer programming for learning directed acyclic graphs from continuous data},
  author={Manzour, Hasan and K{\"u}{\c{c}}{\"u}kyavuz, Simge and Wu, Hao-Hsiang and Shojaie, Ali},
  journal={INFORMS journal on optimization},
  volume={3},
  number={1},
  pages={46--73},
  year={2021},
  publisher={INFORMS}
}

@book{pearl2009causality,
  title={Causality},
  author={Pearl, Judea},
  year={2009},
  publisher={Cambridge university press}
}

@article{cowls2015causation,
  title={Causation, correlation, and big data in social science research},
  author={Cowls, Josh and Schroeder, Ralph},
  journal={Policy \& Internet},
  volume={7},
  number={4},
  pages={447--472},
  year={2015},
  publisher={Wiley Online Library}
}

@inproceedings{duong2022bivariate,
  title={Bivariate causal discovery via conditional divergence},
  author={Duong, Bao and Nguyen, Thin},
  booktitle={Conference on Causal Learning and Reasoning},
  pages={236--252},
  year={2022},
  organization={PMLR}
}

@article{lauritzen1988local,
  title={Local computations with probabilities on graphical structures and their application to expert systems},
  author={Lauritzen, Steffen L and Spiegelhalter, David J},
  journal={Journal of the Royal Statistical Society: Series B (Methodological)},
  volume={50},
  number={2},
  pages={157--194},
  year={1988},
  publisher={Wiley Online Library}
}

@article{sanchez2022diffusion,
  title={Diffusion models for causal discovery via topological ordering},
  author={Sanchez, Pedro and Liu, Xiao and O'Neil, Alison Q and Tsaftaris, Sotirios A},
  journal={arXiv preprint arXiv:2210.06201},
  year={2022}
}

@inproceedings{rolland2022score,
  title={Score matching enables causal discovery of nonlinear additive noise models},
  author={Rolland, Paul and Cevher, Volkan and Kleindessner, Matth{\"a}us and Russell, Chris and Janzing, Dominik and Sch{\"o}lkopf, Bernhard and Locatello, Francesco},
  booktitle={International Conference on Machine Learning},
  pages={18741--18753},
  year={2022},
  organization={PMLR}
}

@article{chow1968approximating,
  title={Approximating discrete probability distributions with dependence trees},
  author={Chow, CKCN and Liu, Cong},
  journal={IEEE transactions on Information Theory},
  volume={14},
  number={3},
  pages={462--467},
  year={1968},
  publisher={IEEE}
}

@inproceedings{jaakkola2010learning,
  title={Learning {B}ayesian network structure using {LP} relaxations},
  author={Jaakkola, Tommi and Sontag, David and Globerson, Amir and Meila, Marina},
  booktitle={Proceedings of the thirteenth international conference on artificial intelligence and statistics},
  pages={358--365},
  year={2010},
  organization={JMLR Workshop and Conference Proceedings}
}

\appendix

\section{Full proof of Theorem 1}



This appendix provides the complete proof of Theorem 1, which characterizes the reduction in admissible parent-set configurations induced by DECO preprocessing.

\begin{proof}In the unconstrained model, each node $v$ can select any subset of the remaining $D-1$ variables. By the binomial theorem, the number of candidate configurations for a single node is $\sum_{k=0}^{D-1} \binom{D-1}{k} = 2^{D-1}$. Across all $D$ nodes, the unrestricted parent-set configuration space is therefore:$$\Omega_{\text{full}} = \left(2^{D-1}\right)^D = 2^{D(D-1)}$$In the DECO framework, the candidate parent pool is restricted to $\mathcal{P}_v$, containing $m_v$ allowable parents. The number of allowable parent sets in $\mathcal{Q}_v$ becomes $\sum_{k=0}^{m_v} \binom{m_v}{k} = 2^{m_v}$. Taking the Cartesian product across all nodes yields an admissible parent-set configuration space prior to global acyclicity enforcement:$$\Omega_{\text{DECO}} = \prod_{v \in V} 2^{m_v} = 2^{\sum_{v \in V} m_v}$$
Under the assumption $\bar{m}\le\left(1-\alpha\right)\left(D-1\right)$, equivalently
$\sum_{v \in V} m_v \le D(1-\alpha)(D-1)$, the ratio of the two configuration spaces satisfies
$$\frac{\Omega_{\text{DECO}}}{\Omega_{\text{full}}} = \frac{2^{\sum_{v \in V} m_v}}{2^{D(D-1)}} \le \frac{2^{D(1-\alpha)(D-1)}}{2^{D(D-1)}} = 2^{-\alpha D(D-1)}$$Thus, a linear reduction in the average admissible parent-pool size yields an exponential reduction in the parent-set configuration space considered prior to global acyclicity enforcement.
\end{proof}

\section{Complete MILP formulation}
For completeness, we provide the integer programming formulation used for downstream exact optimization. This formulation follows the standard parent-set-based Bayesian network structure-learning framework of \citet{jaakkola2010learning}, \citet{cussens2011bayesian}, and \citet{bartlett2017integer}; DECO modifies the admissible parent-set collections supplied to the optimizer.




For each $v\in V$ let  $\mathcal{P}_v$ and $\mathcal{Q}_v$ denote the admissible parent pool constructed by DECO and its power set respectively. For each $Q\in\mathcal{Q}_v$, define binary variables: 

$$z_{v,Q}=\begin{cases}
    1 & \text{if } Q \text{ is selected as the parent set of } v\\
    0 & \text{otherwise}
\end{cases}$$

The restricted parent-set MILP is,

$$\max_{z}\quad 
\sum_{v\in V}\sum_{Q\in\mathcal{Q}_v}
\mathrm{Score}(v,Q)z_{v,Q}$$
subject to
\begin{align*}
    \sum_{Q\in\mathcal{Q}_v}z_{v,Q}=1,&{}\quad \forall v\in V\\
\sum_{v\in C}{\sum_{\substack{
\mathcal{Q}\in\mathcal Q_v\\ \mathcal{Q}\cap C=\emptyset}}}
 z_{v,\mathcal{Q}} \geq 1,&\quad \forall C\subseteq V, \left|C\right|\geq 2\\
 z_{v,\mathcal{Q}}\in \{0,1\}, &\quad\forall v\in V, \mathcal{Q}\in \mathcal{Q}_v.
\end{align*}

The first constraint selects exactly one parent set for each node, while the cluster constraints enforce acyclicity. The violated cluster constraints are added iteratively following the cutting-plane procedure of \citet{cussens2011bayesian}.


%



\section{Additional experimental results}

\subsection{Threshold sensitivity analyses}

The main text (Figure~\ref{trade}) illustrates the effect of the MI threshold, $\delta_{MI}$, for a representative ER-DAG setting. Here, we extend the sensitivity analysis across different graph sizes $n$ and sample-size-to-dimension ratios $N/n$ to assess whether the observed trade-off between skeleton preservation and search-space reduction persists across settings. Figure~6 summarizes the results for discrete ER-DAGs with $p=0.05$.
 

\begin{figure}[ht!]
    \centering
    \includegraphics[width=1\linewidth]{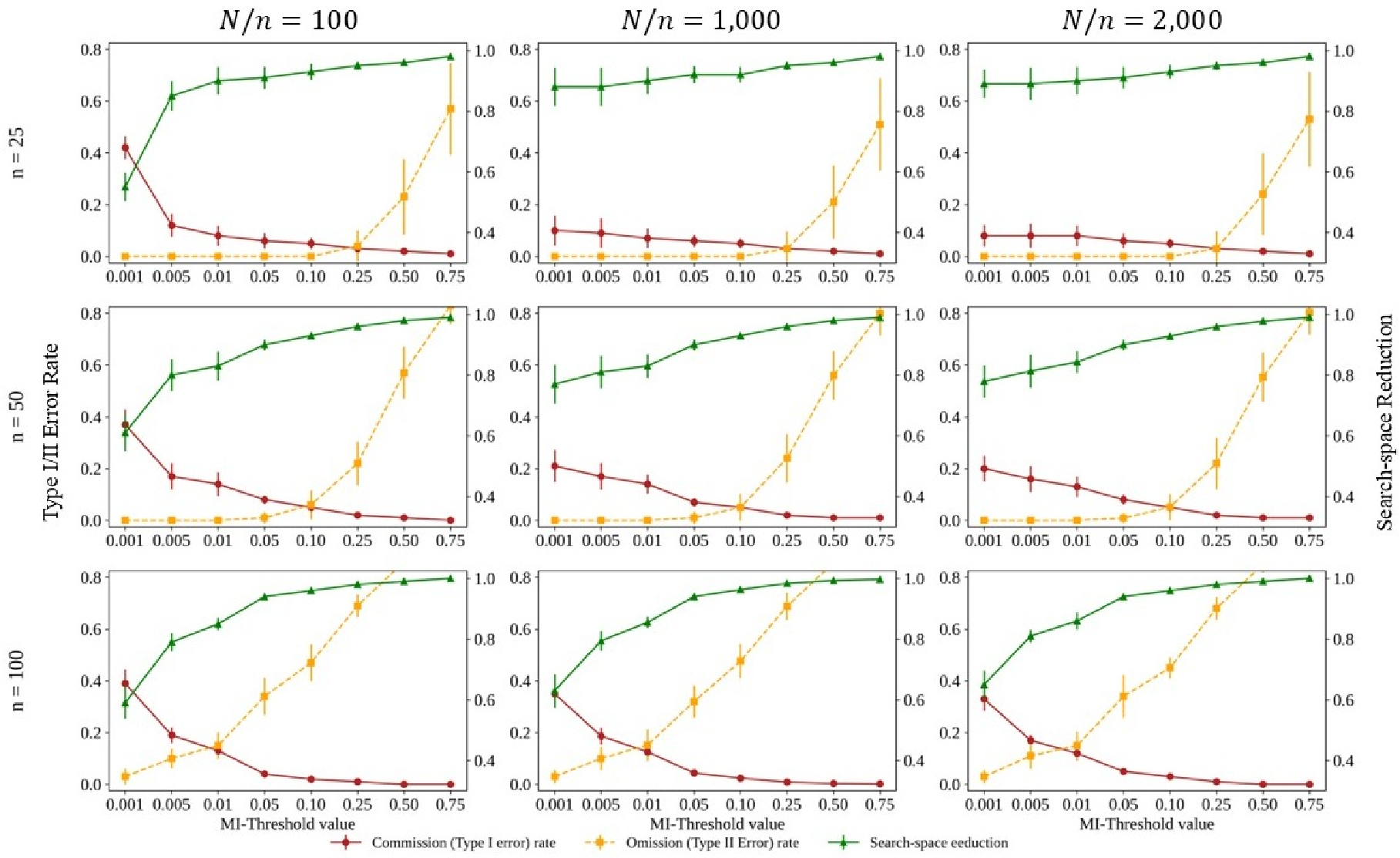}
    \caption{Sensitivity of MI-based dependency screening to $\delta_{MI}$ across graph sizes $n$ and sample size to graph size ratio $N/n$ for discrete Erdős–Rényi DAGs with $p=0.05$. Increasing $\delta_{MI}$ generally decreases commission and increases search-space reduction, but eventually increases omission, illustrating the trade-off between structural preservation and pruning. 
    }
    \label{fig:placeholder}
\end{figure}

\subsection{DAG recovery results}

We compared the SHD of different DAG learning methods, including DECO, in the main text (Figure~\ref{fig:box}). Here, we provide the complete numerical results, including Structural Hamming Distance (SHD), False Discovery Rate (FDR), True Positive Rate (TPR), and False Positive Rate (FPR), across the evaluated discrete and continuous networks and sample sizes.


\setlength{\LTcapwidth}{\textwidth}    
\begin{longtable}{ |c|c|c|c|c|c| }    
    \caption{Comparison of DAG learning performance on benchmark discrete networks} \label{tab:l-xl}

   \\ \hline
    
    Data size & Method & SHD $\downarrow$ & FDR $\downarrow$ & TPR $\uparrow$ & FPR $\downarrow$ \\
    \hline
    \endfirsthead
    \multicolumn{6}{c}{%
    \tablename\ \thetable{} -- continued
}\\

\hline
Data size & Method & SHD $\downarrow$ & FDR $\downarrow$
          & TPR $\uparrow$ & FPR $\downarrow$ \\
\hline

\endhead


\hline
\multicolumn{6}{r}{Continued on next page} \\

\endfoot


\hline

\endlastfoot

\multicolumn{6}{|c|}{Network: Asia (Nodes: 8, Arcs: 8)}\\*
    \hline
    \multirow{5}{*}{1000} & PC & 11.36 $\pm$ 0.975 &  0.88 $\pm$ 0.065 & 0.12 $\pm$ 0.070 &  0.14 $\pm$ 0.033\\*
    & GES & 12.58 $\pm$ 1.002 & 0.89 $\pm$ 0.051 & 0.12 $\pm$ 0.047 &  0.17 $\pm$ 0.028\\*
    & HC & 4.04 $\pm$ 2.191 &  0.35 $\pm$ 0.226 & 0.59 $\pm$ 0.180 &  0.06 $\pm$ 0.040\\
    & GOBNILP & 2.82 $\pm$ 1.506 & 0.24 $\pm$ 0.134 &  0.74 $\pm$ 0.098 & 0.04 $\pm$ 0.026\\
    & DECO & \textbf{1.00 $\pm$ 0.283} & \textbf{0.01 $\pm$ 0.030} &  \textbf{0.88 $\pm$ 0.035} & \textbf{0.00 $\pm$ 0.005}\\
     \hline
     
    \multirow{4}{*}{5000} & PC & 12.24 $\pm$ 0.709 & 0.87 $\pm$ 0.032 & 0.17 $\pm$ 0.060 & 0.18 $\pm $ 0.026\\
    & GES & 12.54 $\pm$ 0.699 & 0.90 $\pm$ 0.019 & 0.13 $\pm$ 0.025 & 0.19 $\pm$  0.028\\
    & HC & 6.64 $\pm$ 2.133 & 0.64 $\pm$ 0.208 & 0.37  $\pm$  0.190 & 0.12 $\pm$  0.043 \\
    & GOBNILP & 2.14$\pm$ 1.327 & 0.19$\pm$ 0.133 & 0.78$\pm$ 0.112 & 0.03$\pm$ 0.026\\
    & DECO & \textbf{0.88 $\pm$ 0.382} & \textbf{0.01 $\pm$ 0.034} & \textbf{0.89 $\pm$ 0.043} & \textbf{0.00 $\pm$ 0.006}\\
    \hline
   \multirow{4}{*}{10000} & PC & 12.86  $\pm$ 0.721 & 0.88 $\pm$ 0.031 &  0.15  $\pm$ 0.053 &    0.19 $\pm$ 0.023\\
    & GES & 12.72 $\pm$ 0.75& 0.90 $\pm$ 0.03 & 0.13 $\pm$ 0.03 & 0.20 $\pm$ 0.02\\
    & HC & 7.02 $\pm$ 2.867 & 0.60 $\pm$ 0.232 & 0.44 $\pm$ 0.200 & 0.12 $\pm$ 0.053\\
    & GOBNILP & 1.34$\pm$ 0.710 & 0.14$\pm$ 0.068 & 0.84$\pm$ 0.078 & 0.02$\pm$ 0.012\\
    & DECO & \textbf{0.58 $\pm$ 0.533} & \textbf{0.02 $\pm$ 0.040} & \textbf{0.93 $\pm$ 0.062} & \textbf{0.00 $\pm$ 0.007}\\
    \hline
   
    \multicolumn{6}{|c|}{Network: Sachs (Nodes: 11, Arcs: 17)}\\
    \hline
    \multirow{4}{*}{1000} & PC & 23.9 $\pm$ 1.723 & 0.81 $\pm$ 0.060 & 0.19 $\pm$ 0.070 &  0.15 $\pm$ 0.022\\*
    & GES & 23.84 $\pm$ 1.332 & 0.88 $\pm$ 0.046 & 0.10 $\pm$ 0.041 &  0.12 $\pm$ 0.017\\
    & HC & 9.94 $\pm$ 2.716 & 0.46 $\pm$ 0.151 & 0.45 $\pm$ 0.141 & 0.07 $\pm$ 0.022\\
    & GOBNILP & \textbf{7.00 $\pm$ 1.709} & \textbf{0.37 $\pm$ 0.101} & \textbf{0.59 $\pm$ 0.101} & \textbf{0.06 $\pm$ 0.017}\\
    & DECO & 9.62 $\pm$  2.244 & 0.49 $\pm$ 0.133 & 0.44 $\pm$  0.122 &    0.08 $\pm$ .019 \\
     \hline

    \multirow{4}{*}{5000} & PC & 26.14  $\pm$ 1.250  & 0.84 $\pm$ 0.050  & 0.23 $\pm$ 0.078 & 0.21 $\pm$ 0.020\\
    & GES & 26.84 $\pm$ 1.255 & 0.88 $\pm$ 0.046 & 0.13 $\pm$ 0.052 & 0.17 $\pm$ 0.0171 \\
    & HC & 8.44 $\pm$ 1.846 & 0.47 $\pm$ 0.094 & 0.55 $\pm$ 0.097 & 0.09 $\pm$ 0.018\\
    & GOBNILP & \textbf{5.02 $\pm$ 1.881} & \textbf{0.3 $\pm$ 0.111} & \textbf{0.7 $\pm$ 0.111} & \textbf{0.05 $\pm$ 0.020}\\
    & DECO & 7.68 $\pm$  1.618 & 0.42 $\pm$ 0.101 & 0.55 $\pm$ 0.095 & 0.07 $\pm$ 0.017\\
    \hline
   \multirow{4}{*}{10000} & PC & 27.38 $\pm$ 1.384 & 0.83 $\pm$ 0.048 & 0.25 $\pm$ 0.077 & 0.22 $\pm$ 0.016\\
    & GES & 27.62 $\pm$ 0.978 & 0.85 $\pm$ 0.035 & 0.18 $\pm$ 0.0489 & 0.19 $\pm$ 0.014\\
    & HC & 9.24  $\pm$ 2.196 & 0.51 $\pm$ 0.120 & 0.51 $\pm$ 0.116 &  0.10 $\pm$ 0.024 \\
    & GOBNILP & \textbf{5.08 $\pm$ 1.978} & \textbf{0.30 $\pm$ 0.116} & \textbf{0.70 $\pm$ 0.116} & \textbf{0.05 $\pm$ 0.021}\\
    & DECO & 7.42 $\pm$ 1.961 & 0.40 $\pm$ 0.123 & 0.56 $\pm$ 0.115 & 0.07 $\pm$ 0.021\\

    \hline

    \multicolumn{6}{|c|}{Network: Child (Nodes: 20, Arcs: 25)}\\
    \hline
    \multirow{4}{*}{1000} & PC & 43.38    $\pm$ 1.742 &  0.89 $\pm$ 0.026 &0.14   $\pm$ 0.032 &  0.08 $\pm$ 0.006\\
    & GES & 48.10    $\pm$ 2.759  &   0.90 $\pm$ 0.034  &  0.15  $\pm$ 0.055   &     0.09$\pm$ 0.010 \\
    & HC & 14.64    $\pm$ 2.696  &  0.42 $\pm$ 0.088  & 0.55   $\pm$ 0.079  &    0.03 $\pm$ 0.007 \\
    & GOBNILP & 17.04 $\pm$ 1.777 & 0.56 $\pm$ 0.070 & 0.44 $\pm$ 0.067 & 0.11$\pm$ 0.004\\
    & DECO & \textbf{11.30 $\pm$ 1.616} & \textbf{0.16 $\pm$ 0.068} & \textbf{0.57 $\pm$ 0.058} & \textbf{0.01 $\pm$ 0.004}\\
    \hline
    \multirow{4}{*}{5000} & PC  & 50.74   $\pm$ 1.968  &  0.91  $\pm$ 0.030 & 0.14    $\pm$ 0.050 & 0.10 $\pm$ 0.007\\
    & GES & 60.12  $\pm$  2.790  &  0.89 $\pm$ 0.023  &0.24   $\pm$ 0.057  &0.14 $\pm$ 0.009\\
    & HC & 9.78  $\pm$ 3.875 &   0.33 $\pm$ 0.118  & \textbf{0.69  $\pm$ 0.097} &  0.03 $\pm$ 0.010 \\
    & GOBNILP & 12.66 $\pm$ 1.012 & 0.31 $\pm$ 0.040 & 0.59 $\pm$ 0.039 & 0.11 $\pm$ 0.003\\
    & DECO & \textbf{9.68 $\pm$ 1.174} & \textbf{0.21 $\pm$ 0.035}  & 0.66 $\pm$ 0.038 & \textbf{0.01 $\pm$ 0.002}\\
    \hline
    \multirow{4}{*}{10000} & PC & 53.72    $\pm$ 1.970 &  0.91  $\pm$ 0.027 &  0.16    $\pm$ 0.050 &   0.11 $\pm$ 0.010 \\
    & GES & 65.88   $\pm$ 2.268 &    0.90 $\pm$ 0.022 & 0.24  $\pm$ 0.060 &    0.15 $\pm$ 0.010\\
    & HC & 10.86   $\pm$ 2.408 &   0.38 $\pm$ 0.069 &  \textbf{0.68  $\pm$ 0.058} &     0.03 $\pm$ 0.006 \\
    & GOBNILP & 12.48 $\pm$ 0.922 & 0.40 $\pm$ 0.037 & 0.50 $\pm$ 0.037 & 0.11 $\pm$ 0.003\\
    & DECO & \textbf{9.3    $\pm$ 1.136} &  \textbf{0.24 $\pm$ 0.032}  & 0.67  $\pm$ 0.045 &   \textbf{0.02 $\pm$ 0.002} \\
    \hline
    \multicolumn{6}{|c|}{Network: Insurance (Nodes: 27, Arcs: 52)}\\
    \hline
    \multirow{4}{*}{10000} & PC & 101.34     $\pm$ 2.909 &    0.86  $\pm$ 0.019 & 0.21 $\pm$ 0.028 &  0.11 $\pm$ 0.005\\
    & GES & 129.00  $\pm$ 3.578 &       0.93   $\pm$  0.020 & 0.13 $\pm$ 0.041 &  0.15 $\pm$  0.007 \\
    & HC & 28.96 $\pm$ 5.352 & 0.36 $\pm$ 0.083 & 0.57 $\pm$ 0.065 & 0.03 $\pm$ 0.007\\
    & GOBNILP & 39.74 $\pm$ 1.180 & 0.46 $\pm$ 0.035 & 0.27 $\pm$ 0.017 & 0.02 $\pm$ 0.001\\
    & DECO & \textbf{22.50 $\pm$ 0.671} & \textbf{0.21 $\pm$ 0.023} & \textbf{0.61 $\pm$ 0.013} & \textbf{0.01 $\pm$ 0.002} \\
    \hline
    \multicolumn{6}{|c|}{Network: Water (Nodes: 32, Arcs: 66)}\\
    \hline
    \multirow{4}{*}{10000} & PC & 100.58 $\pm$ 3.864 & 0.91 $\pm$ 0.012 & 0.02 $\pm$ 0.034 & 0.09$\pm$ 0.004\\*
    & GES & 102.22 $\pm$ 2.250 & 0.96 $\pm$ 0.017 & 0.03 $\pm$ 0.012 & 0.05 $\pm$ 0.003 \\*
    & HC & 58.70 $\pm$ 4.847 & 0.58  $\pm$ 0.102 & 0.22 $\pm$ 0.051 & 0.02$ \pm$ 0.004\\
    & GOBNILP & 59.18 $\pm$ 2.755 & 0.42 $\pm$ 0.064 & 0.24 $\pm$ 0.032 & 0.02 $\pm$ 0.003\\
    & DECO & \textbf{53.66 $\pm$ 3.198} & \textbf{0.46 $\pm$ 0.042} & \textbf{0.38 $\pm$ 0.044} & \textbf{0.02 $\pm$ 0.004} \\
    \hline
    \multicolumn{6}{|c|}{Network: Hailfinder (Nodes: 56, Arcs: 66)}\\
    \hline
    \multirow{4}{*}{1000} & PC & 98.34 $\pm$ 2.055 & 0.97 $\pm$ 0.013 & 0.02 $\pm$ 0.009 & 0.02 $\pm$ 0.001\\
     & GES & 162.14  $\pm$ 6.642 & 0.96 $\pm$ 0.015 &  0.07 $\pm$ 0.024 &  0.04 $\pm$ 0.002\\  
    & HC & \textbf{13.06    $\pm$ 3.295} &  0.37 $\pm$ 0.107 & \textbf{0.59    $\pm$ 0.087} &   0.03 $\pm$ 0.008 \\
    & GOBNILP & 55.86 $\pm$ 3.371 & 0.62 $\pm$ 0.039 & 0.31 $\pm$ 0.032 & \textbf{0.01 $\pm$ 0.001}  \\
    & DECO & 39.95  $\pm$ 4.753 & \textbf{0.35 $\pm$ 0.060} & 0.52 $\pm$ 0.045 & \textbf{0.01 $\pm$ 0.001}  \\
    \hline
    \multirow{4}{*}{5000} & PC & 102.80 $\pm$ 1.876 & 0.97 $\pm$ 0.012 & 0.02 $\pm$ 0.009 & 0.02 $\pm$ 0.001\\
   & GES & 188.63 $\pm$ 6.304 & 0.95 $\pm$ 0.018 & 0.10 $\pm$ 0.037 & 0.05 $\pm$ 0.002\\
    & HC & \textbf{10.24     $\pm$ 3.409} &  0.35 $\pm$ 0.105 & \textbf{0.70  $\pm$ 0.083} &    0.03 $\pm$ 0.010 \\
    & GOBNILP & 53.68 $\pm$ 4.017 & 0.60 $\pm$ 0.043 & 0.34 $\pm$ 0.036 & \textbf{0.01 $\pm$ 0.001}\\
    & DECO & 38.21  $\pm$ 4.270 & \textbf{0.30 $\pm$ 0.044} & 0.57 $\pm$ 0.037 & \textbf{0.01 $\pm$ 0.001}\\
    \hline
    \multirow{4}{*}{10000} & PC & 103.08 $\pm$ 2.048 & 0.97 $\pm$  0.009 & 0.02 $\pm$ 0.007 & 0.02 $\pm$ 0.001\\*
    & GES & 204.6 $\pm$ 5.03 & 0.77 $\pm$ 0.425 & 0.12 $\pm$ 0.026 & 0.05 $\pm$ 0.000\\*
    & HC & \textbf{10.48     $\pm$ 2.988} &  0.37  $\pm$ 0.085 &  \textbf{0.70    $\pm$ 0.069} &    0.03 $\pm$ 0.008\\
    & GOBNILP & 53.28 $\pm$ 4.336& 0.60 $\pm$ 0.044 & 0.33 $\pm$ 0.036 & 0.01 $\pm$ 0.001\\
    & DECO & 37.10 $\pm$ 6.068 & \textbf{0.28 $\pm$ 0.063} & 0.63 $\pm$ 0.095 & \textbf{0.00 $\pm$ 0.001}\\
    \hline
\end{longtable}

\setlength{\LTcapwidth}{\textwidth}  
\begin{longtable}{ |c|c|c|c|c|c| }
    \caption{Comparison of DAG learning performance on Erd\H{o}s-R\'{e}nyi Graphs and discrete data} \label{ER-disc}
   \\ \hline
    

    Data size & Method & SHD $\downarrow$ & FDR $\downarrow$ & TPR $\uparrow$ & FPR $\downarrow$ \\
    \hline
    \endfirsthead
    
\multicolumn{6}{c}{%
    \tablename\ \thetable{} -- continued
}\\

\hline
Data size & Method & SHD $\downarrow$ & FDR $\downarrow$
          & TPR $\uparrow$ & FPR $\downarrow$ \\
\hline

\endhead


\hline
\multicolumn{6}{r}{Continued on next page} \\

\endfoot


\hline

\endlastfoot
    
    \multicolumn{6}{|c|}{Nodes: 25, Edge probability: 0.15}\\
    \hline

      \multirow{4}{*}{1000}   &  PC & 19.38 $\pm$ 5.779 &  0.3  $\pm$ 0.085 &  0.74  $\pm$ 0.114 &  0.03 $\pm$ 0.007\\
         &  GES & 16.28 $\pm$ 5.923 & 0.27 $\pm$ 0.087 & \textbf{0.82 $\pm$ 0.088} & 0.02 $\pm$ 0.009\\
         & HC & 19.82 $\pm$ 6.069 & 0.31 $\pm$ 0.099 & 0.63 $\pm$ 0.092 & 0.02 $\pm$ 0.008\\
         & GOBNILP & 31.80 $\pm$ 5.437 & 0.43 $\pm$ 0.083 & 0.30 $\pm$ 0.054 & \textbf{0.02 $\pm$ 0.003}\\
         & DECO & \textbf{13.96 $\pm$ 5.269}& \textbf{0.23 $\pm$ 0.094} & 0.69 $\pm$ 0.094& 0.02 $\pm$ 0.007 \\
      \hline
      \multirow{4}{*}{5000}   &  PC & 15.08 $\pm$ 6.347 & 0.25 $\pm$ 0.070 & 0.89  $\pm$ 0.087 & 0.02 $\pm$ 0.009\\*
         &  GES & 16.30 $\pm$ 11.988 &  0.28 $\pm$ 0.120 & \textbf{0.91 $\pm$ 0.093} & 0.03 $\pm$ 0.020\\*
         & HC & 16.76 $\pm$ 7.663 & 0.3 $\pm$ 0.121 &  0.74 $\pm$ 0.096 & 0.03 $\pm$ 0.013\\
         & GOBNILP & 32.24 $\pm$ 5.965 & 0.41 $\pm$ 0.082 & 0.30 $\pm$ 0.051 & 0.02 $\pm$ 0.003\\
         & DECO & \textbf{10.54 $\pm$ 4.050} & \textbf{0.20 $\pm$ 0.095} & 0.78 $\pm$ 0.077 & \textbf{0.01 $\pm$ 0.008}\\
      \hline
      \multirow{4}{*}{10000}   &  PC & 15.97 $\pm$ 5.086 & 0.25 $\pm$ 0.055 & 0.9  $\pm$ 0.066 & 0.03 $\pm$ 0.008\\
         &  GES & 18.45 $\pm$ 10.437 & 0.31 $\pm$ 0.095 & \textbf{0.92 $\pm$ 0.070} & 0.04 $\pm$ 0.017\\
         & HC & 13.16 $\pm$ 6.590 &  0.23 $\pm$ 0.107 & 0.79 $\pm$ 0.077 & 0.02 $\pm$ 0.011\\
         & GOBNILP & 34.24 $\pm$ 5.252 & 0.43 $\pm$ 0.102 & 0.28 $\pm$ 0.052 & \textbf{0.02 $\pm$ 0.004}\\
         & DECO & \textbf{12.22 $\pm$ 2.393} & \textbf{0.23 $\pm$ 0.080} & 0.75 $\pm$ 0.039 & 0.02 $\pm$ 0.007\\
      \hline

      \multicolumn{6}{|c|}{Nodes: 25, Edge probability: 0.25}\\
    \hline
      \multirow{4}{*}{1000}   &  PC & 53.11 $\pm$ 8.162 & 0.41 $\pm$ 0.074 & 0.45 $\pm$ 0.062 & 0.05 $\pm$ 0.010\\
      & GES & 53.04 $\pm$ 11.766 & 0.41 $\pm$ 0.111 & \textbf{0.54 $\pm$ 0.098} & 0.06 $\pm$ 0.019\\
      & HC & 47.26 $\pm$ 10.458 & 0.32 $\pm$ 0.107 & 0.42 $\pm$ 0.009 & 0.03 $\pm$ 0.001\\
      & GOBNILP & 59.98 $\pm$ 10.029 & 0.38 $\pm$ 0.073 & 0.20 $\pm$ 0.038 & 0.02 $\pm$ 0.008\\
      & DECO & \textbf{43.78 $\pm$ 6.927} & \textbf{0.24 $\pm$ 0.082} & 0.43 $\pm$ 0.061 & \textbf{0.02 $\pm$ 0.007}\\
      \hline
      \multirow{4}{*}{5000}   &  PC & 44.40 $\pm$ 8.542 &  0.36 $\pm$ 0.068 &  0.62 $\pm$ 0.083 &  0.05 $\pm$ 0.012\\
      & GES & 51.18 $\pm$ 12.912& 0.41 $\pm$ 0.078 &  \textbf{0.71 $\pm$ 0.077} &  0.07 $\pm$ 0.022\\
      & HC & \textbf{35.46  $\pm$ 7.981} & 0.26 $\pm$ 0.094 &  0.58 $\pm$ 0.074 &  0.03 $\pm$ 0.012\\
      & GOBNILP & 58.62 $\pm$ 7.389 & 0.35 $\pm$ 0.100 & 0.21 $\pm$ 0.037 & 0.02 $\pm$ 0.005\\
      & DECO & 43.47  $\pm$ 6.365 &   \textbf{0.26 $\pm$ 0.073} & 0.44 $\pm$ 0.052 & 0.02 $\pm$ 0.006\\
      \hline
      \multirow{4}{*}{10000}   &  PC & 44.16 $\pm$ 9.501 &  0.36 $\pm$ 0.057 &  0.65 $\pm$ 0.078 & 0.05 $\pm$ 0.011\\
      & GES & 59.51 $\pm$ 15.663 &  0.46 $\pm$ 0.093 & \textbf{0.73 $\pm$ 0.090} & 0.09 $\pm$ 0.029\\
      & HC & \textbf{32.65 $\pm$ 9.967} & \textbf{0.25 $\pm$ 0.111} & 0.64 $\pm$ 0.087 & 0.03 $\pm$ 0.015\\
      & GOBNILP & 58.58 $\pm$ 7.787 & 0.36 $\pm$ 0.093 & 0.21 $\pm$ 0.033 & \textbf{0.02 $\pm$ 0.004}\\
      & DECO & 40.5 $\pm$ 6.64 & 0.27 $\pm$ 0.088 & 0.48 $\pm$ 0.088 & 0.02 $\pm$ 0.007\\
      \hline
      \multicolumn{6}{|c|}{Nodes: 50, Edge probability: 0.05}\\*
    \hline
      \multirow{4}{*}{10000} & PC & 28.70 $\pm$ 5.248 & 0.35 $\pm$ 0.070 & 0.90 $\pm$ 0.059 & 0.01 $\pm$ 0.002\\*
      & GES & 15.73 $\pm$ 4.399 & 0.27 $\pm$ 0.058 &  \textbf{0.97 $\pm$ 0.041} &  0.01 $\pm$ 0.002\\*
      & HC & 17.23 $\pm$ 5.976 & 0.26 $\pm$ 0.077 & 0.81 $\pm$ 0.058 & 0.01 $\pm$ 0.003\\*
      & GOBNILP & 17.76 $\pm$ 4.488 & 0.53 $\pm$ 0.080 & 0.48 $\pm$ 0.080 & 0.01 $\pm$ 0.002\\		
      & DECO & \textbf{9.80	$\pm$ 2.280} & \textbf{0.14 $\pm$ 0.030} &	0.86 $\pm$ 0.041 &	\textbf{0.01 $\pm$ 0.001}\\
      \hline
    
\end{longtable}

\setlength{\LTcapwidth}{\textwidth}  
\begin{longtable}{ |c|c|c|c|c|c| }
    \caption{Comparison of DAG learning performance on Erd\H{o}s-R\'{e}nyi Graphs and continuous (Non-linear Gaussian) data}
    \\ \hline
    

    Data size & Method & SHD $\downarrow$ & FDR $\downarrow$ & TPR $\uparrow$ & FPR $\downarrow$ \\
    \hline
    \endfirsthead
    
\multicolumn{6}{c}{%
    \tablename\ \thetable{} -- continued
}\\

\hline
Data size & Method & SHD $\downarrow$ & FDR $\downarrow$
          & TPR $\uparrow$ & FPR $\downarrow$ \\
\hline

\endhead


\hline
\multicolumn{6}{r}{Continued on next page} \\

\endfoot


\hline

\endlastfoot
    
    \multicolumn{6}{|c|}{Nodes: 25, Edge probability: 0.15}\\
    \hline
   \multirow{6}{*}{1000} & PC & 21.53 $\pm$6.492 & 0.34 $\pm$ 0.090 & 0.71 $\pm$ 0.100 & 0.03 $\pm$ 0.009\\
      & GES & 21.90 $\pm$ 10.239 & 0.34 $\pm$ 0.110 & \textbf{0.81 $\pm$ 0.099} & 0.04 $\pm$ 0.015\\
      & HC & 29.10 $\pm$ 6.519 & 0.46 $\pm$ 0.080 & 0.70 $\pm$ 0.081 & 0.05 $\pm$ 0.008\\
      & NOTEARS & 25.27 $\pm$ 3.844 & \textbf{0.00 $\pm$ 0.013} & 0.46 $\pm$ 0.043 & \textbf{0.00 $\pm$ 0.001}\\
      & GOLEM & 21.30 $\pm$ 4.518 & 0.02 $\pm$ 0.026 & 0.54 $\pm$ 0.049 & \textbf{0.00 $\pm$ 0.001} \\
      & GOBNILP & 43.42 $\pm$ 7.236 & 0.28 $\pm$ 0.414 & 0.01 $\pm$ 0.018 & \textbf{0.00 $\pm$ 0.001}\\
      & DECO & \textbf{20.67 $\pm$ 3.682} & 0.40 $\pm$ 0.030 & 0.67 $\pm$ 0.010 & 0.03 $\pm$ 0.007     \\ 
      \hline
      \multirow{6}{*}{5000} & PC & 18.72 $\pm$ 4.054 & 0.34 $\pm$ 0.072 & 0.76 $\pm$ 0.095 & 0.03 $\pm$ 0.007\\*
      & GES & 24.36 $\pm$ 13.328 &  0.39 $\pm$ 0.149 &  \textbf{0.85 $\pm$ 0.125} &  0.05 $\pm$ 0.024\\*
      & HC & 43.82 $\pm$ 13.750 & 0.55 $\pm$ 0.111 &  0.71 $\pm$ 0.160 &  0.08 $\pm$ 0.027\\*
      & NOTEARS & 35.20 $\pm$ 4.284 & 0.09 $\pm$ 0.095 & 0.38  $\pm$ 0.105 & \textbf{0.01  $\pm$ 0.003} \\
      & GOLEM & 25.86  $\pm$ 5.455 & \textbf{0.02  $\pm$ 0.047} & 0.53  $\pm$ 0.089 & 0.01  $\pm$ 0.007 \\
      & GOBNILP & 32.72 $\pm$ 6.690 & 0.33 $\pm$ 0.114 & 0.29 $\pm$ 0.077 & 0.01 $\pm$ 0.004\\
      & DECO & \textbf{19.67 $\pm$ 5.676} & 0.34 $\pm$ 0.089 & 0.65 $\pm$ 0.083 & 0.03 $\pm$ 0.007\\
      \hline
      \multirow{6}{*}{10000} & PC & 19.40 $\pm$ 7.011 & 0.33 $\pm$ 0.078 &  0.78 $\pm$ 0.113 &   0.03 $\pm$ 0.010\\
      & GES & 27.85 $\pm$ 14.115 &  0.40 $\pm$ 0.117 &  \textbf{0.86 $\pm$ 0.090} &   0.05 $\pm$ 0.024\\
      & HC & 52.25 $\pm$ 10.321 & 0.58 $\pm$ 0.059 &   0.80 $\pm$ 0.091 &  0.10 $\pm$ 0.018\\
      & NOTEARS & 26.53 $\pm$ 7.546  &     0.02 $\pm$ 0.031 &  0.40 $\pm$ 0.081 &  0.01 $\pm$  0.005  \\
      & GOLEM & 20.13 $\pm$ 6.710 &   \textbf{0.01   $\pm$ 0.005} &  0.55 $\pm$ 0.075 & \textbf{0.01  $\pm$ 0.003} \\
      & GOBNILP & 31.64 $\pm$ 6.039 & 0.39 $\pm$ 0.095 & 0.31 $\pm$ 0.064 & 0.02 $\pm$ 0.004\\
      & DECO & \textbf{18.00 $\pm$ 3.122} & 0.08 $\pm$ 0.023 & 0.68 $\pm$ 0.051 & 0.03 $\pm$ 0.001\\
      \hline
      \multicolumn{6}{|c|}{Nodes: 25, Edge probability: 0.05}\\
      \hline

      \multirow{6}{*}{10000}& PC & 14.7  $\pm$ 3.864 & 0.55 $\pm$ 0.106 & 0.88 $\pm$ 0.103 & 0.03 $\pm$ 0.008\\
      & GES & 6.54 $\pm$ 3.721 & 0.4 $\pm$ 0.117 & \textbf{0.95 $\pm$ 0.085} & 0.02$ \pm$ 0.008\\
      & HC & 14.98 $\pm$ 6.516 & 0.57 $\pm$ 0.111 & 0.69 $\pm$ 0.126 & 0.03 $\pm$ 0.011\\
      & NOTEARS & 3.07 $\pm$ 1.931 & \textbf{0.00 $\pm$ 0.000} &   0.82 $\pm$ 0.085 & \textbf{0.00 $\pm$ 0.000} \\
      & GOLEM & \textbf{1.33 $\pm$ 0.943} & 0.00  $\pm$ 0.002 & 0.9  $\pm$ 0.073 & \textbf{0.00 $\pm$  0.000} \\
      & GOBNILP & 8.38 $\pm$ 2.939 & 0.47 $\pm$ 0.145 & 0.45 $\pm$ 0.130 & 0.01 $\pm$ 0.004\\
      & DECO & 2.87 $\pm$ 2.029 & 0.2 $\pm$ 0.141 &  0.84 $\pm$ 0.118 & 0.00  $\pm$ 0.003 \\
      \hline
      \multicolumn{6}{|c|}{Nodes: 50, Edge probability: 0.05}\\
    \hline
   \multirow{6}{*}{10000}& PC & 35.67	$\pm$ 2.887 & 0.41	$\pm$ 0.059 & \textbf{0.86 $\pm$ 0.035} & 0.02 $\pm$ 0.006\\*
   & GES & 20.00 $\pm$ 3.606 &	0.29 $\pm$ 0.012 &	0.73	$\pm$ 0.410 & 0.01 $\pm$ 0.001\\
   & HC & 86.33	$\pm$ 33.005 & 0.64 $\pm$ 0.025 &	0.74	$\pm$ 0.059 & 0.04 $\pm$ 0.014\\
   & NOTEARS & 14.50 $\pm$ 2.778 & \textbf{0.00 $\pm$ 0.001} & 0.73 $\pm$ 0.032 & \textbf{0.00 $\pm$ 0.002}\\
   & GOLEM & 13.63  $\pm$ 4.368 &  \textbf{0.00   $\pm$ 0.001} &  0.77 $\pm$ 0.056 & 0.00 $\pm$ 0.004 \\
   & GOBNILP & 31.62 $\pm$ 14.343 & 0.35 $\pm$ 0.156 & 0.34 $\pm$ 0.150 &  0.01 $\pm$ 0.003 \\
   & DECO & \textbf{11.69 $\pm$ 4.134} & 0.05 $\pm$ 0.002 & 0.79 $\pm$ 0.056& \textbf{0.00 $\pm$ 0.002}\\
   \hline
    
\end{longtable}

\setlength{\LTcapwidth}{\textwidth}  
\begin{longtable}{ |c|c|c|c|c|c| }
    \caption{Comparison of DAG learning performance on Erd\H{o}s-R\'{e}nyi Graphs and continuous (Linear Gaussian) data}
    \\ \hline
    

    Data size & Method & SHD $\downarrow$ & FDR $\downarrow$ & TPR $\uparrow$ & FPR $\downarrow$ \\
    \hline
    \endfirsthead
    
\multicolumn{6}{c}{%
    \tablename\ \thetable{} -- continued
}\\

\hline
Data size & Method & SHD $\downarrow$ & FDR $\downarrow$
          & TPR $\uparrow$ & FPR $\downarrow$ \\
\hline

\endhead


\hline
\multicolumn{6}{r}{Continued on next page} \\

\endfoot


\hline

\endlastfoot
    
    \multicolumn{6}{|c|}{Nodes: 25, Edge probability: 0.15}\\
    \hline
\multirow{6}{*}{1000} & PC & 32.33 $\pm$ 10.162 & 0.41 $\pm$ 0.100 & 0.52 $\pm$ 0.140 & 0.03 $\pm$ 0.008\\
      & GES & 33.35 $\pm$ 30.075 & 0.42 $\pm$ 0.198 & 0.83 $\pm$ 0.146 & 0.06 $\pm$ 0.053\\
      & HC & 78.25 $\pm$ 13.139 & 0.71 $\pm$ 0.066 & 0.64 $\pm$ 0.097 & 0.13 $\pm$ 0.033\\
      & NOTEARS & 4.00 $\pm$ 5.172 &  0.02 $\pm$ 0.038 & 0.94   $\pm$ 0.056 & 0.00 $\pm$  0.004 \\
      & GOLEM & \textbf{1.25 $\pm$ 2.222} & \textbf{0.01 $\pm$ 0.022} & \textbf{0.99 $\pm$ 0.027} & \textbf{0.00 $\pm$ 0.002}\\
      & DECO & 25 $\pm$ 3.611 & 0.12 $\pm$ 0.071 &  0.85 $\pm$  0.013 &  0.01 $\pm$ 0.004\\
      \hline
\end{longtable}
\end{document}